\documentclass[11pt,table]{article}
\usepackage[in]{fullpage}           % set all margins to 1 inch
\usepackage[sc]{titlesec}           % option to set small caps for section titles
\usepackage{microtype}
\usepackage{xurl}                   % url package with extra breakpoints
\usepackage{xcolor}
\usepackage{amsmath,amssymb}
\usepackage{algpseudocode}
\usepackage{algorithm}
\algnewcommand{\LineComment}[1]{\State 
 \textcolor{gray}{\# #1}}
\usepackage{listings}
\usepackage{mathtools}
\usepackage{titlesec}
\usepackage[numbers,sort]{natbib}
\usepackage{booktabs}
\usepackage{multirow}
\usepackage{graphicx}
\usepackage{bbm}
\usepackage{subcaption}
\usepackage{tikz}
\usepackage[font=small,labelfont=bf]{caption}
\usepackage{xspace}
\usepackage{enumitem}
\usepackage{amsthm}

\usepackage[outline]{contour}
\usepackage{ulem}
\usepackage[breaklinks]{hyperref}
\DeclareMathOperator*{\E}{\mathbb{E}}
\DeclareMathOperator{\x}{\mathbf{x}}
\DeclareMathOperator{\X}{\mathcal{X}}
\DeclareMathOperator{\y}{\mathbf{y}}
\DeclareMathOperator{\vv}{\mathbf{v}}
\DeclareMathOperator{\Y}{\mathcal{Y}}
\renewcommand{\k}{\mathbf{k}}
\DeclareMathOperator{\e}{\mathbf{e}}
\DeclareMathOperator{\m}{\boldsymbol{\mu}}
\newcommand{\pt}[1]{\rho_{#1}}
\newcommand{\mt}[1]{\boldsymbol{\mu}_{#1}}
\newcommand{\kl}[2]{D_{KL}\left(#1 \parallel #2\right)}
\newcommand{\N}[2]{\mathcal{N}\left(#1 , #2\right)}
\newcommand{\bc}[1]{#1_c}
\DeclareMathOperator{\R}{\mathbb{R}}
\newcommand{\I}[1]{\boldsymbol{I}}
\newcommand*{\defeq}{\stackrel{\text{def}}{=}}
\newtheorem{theorem}{Theorem}[section]
\newtheorem{proposition}[theorem]{Proposition}
\newcommand{\tidx}[2]{#1_{#2}}
\newcommand{\didx}[2]{#1^{(#2)}}
\renewcommand{\vec}[1]{\boldsymbol{#1}}
\newcommand{\pars}{\theta}
\newcommand{\parsn}{\vec{\pars}}
\newcommand{\parst}[1]{\tidx{\pars}{#1}}
\newcommand{\parsnt}[1]{\tidx{\parsn}{#1}}
\newcommand{\alphat}[1]{\tidx{\alpha}{#1}}
\newcommand{\yt}[1]{\tidx{\y}{#1}}
\newcommand{\constvec}[2]{\vec{#1}}
\newcommand{\0}[1]{\constvec{0}{#1}}
\newcommand{\1}[1]{\constvec{1}{#1}}
\newcommand{\yd}{y}
\newcommand{\ydd}[1]{\didx{\yd}{#1}}
\newcommand{\xdd}[1]{\didx{x}{#1}}
\newcommand{\parsdd}[1]{\didx{\pars}{#1}}
\newcommand{\oh}[2]{\mathbf{e}_{#1}}
\newcommand{\ds}[1]{\{1,#1\}}
\newcommand{\dsd}[2]{\ds{#1}^{#2}}
\newcommand{\ui}[1]{U\ds{#1}}
\titleformat{\paragraph}
{\normalfont\normalsize\bfseries}{\theparagraph}{1em}{}
\titlespacing*{\paragraph}
{0pt}{3.25ex plus 1ex minus .2ex}{1.5ex plus .2ex}

\def\net{\Psi\xspace}
\newcommand{\sender}[2]{p_{_S}\left(#1 \mid #2\right)}
\newcommand{\out}{p_{_O}}
\newcommand{\outn}{\vec{p}_{_O}}
\newcommand{\rec}{p_{_R}}
\newcommand{\inp}{p_{_I}}
\newcommand{\flow}{p_{_F}}
\newcommand{\update}{p_{_U}}
\newcommand{\pred}[1]{\hat{#1}}
\newcommand{\eps}{\vec{\pred{\epsilon}}}

\begin{document}
\title{\textsc{Bayesian Flow Networks}}
\author{Alex Graves, Rupesh Kumar Srivastava, Timothy Atkinson, Faustino Gomez}
\date{
\vspace{-6pt}
\texttt{\{alex,rupesh,timothy,tino\}@nnaisense.com}\\
\vspace{6pt}
NNAISENSE
}
\maketitle

\begin{abstract}
This paper introduces \emph{Bayesian Flow Networks} (BFNs), a new class of generative model in which the parameters of a set of independent distributions are modified with Bayesian inference in the light of noisy data samples, then passed as input to a neural network that outputs a second, interdependent distribution.
Starting from a simple prior and iteratively updating the two distributions yields a generative procedure similar to the reverse process of diffusion models; however it is conceptually simpler in that no forward process is required.
Discrete and continuous-time loss functions are derived for continuous, discretised and discrete data, along with sample generation procedures.
Notably, the network inputs for discrete data lie on the probability simplex, and are therefore natively differentiable, paving the way for gradient-based sample guidance and few-step generation in discrete domains such as language modelling.
The loss function directly optimises data compression and places no restrictions on the network architecture.
In our experiments BFNs achieve competitive log-likelihoods for image modelling on dynamically binarized MNIST and CIFAR-10, and outperform all known discrete diffusion models on the text8 character-level language modelling task\footnote{Code and trained models can be found at \url{https://github.com/nnaisense/bayesian-flow-networks}}.
\end{abstract}
%%%%%%%%%%%%%%%%%%%%%%%%%%%%%%%%%%%%%%%%%%%%%%%%%%%%
\section{Introduction}
Large-scale neural networks have revolutionised generative modelling over the last few years, with an unprecedented ability to capture complex relationships among many variables. 
Building a convincing joint model of all the pixels in a high resolution image, for example, was impossible before the advent of modern generative networks.

Key to the expressive power of most of these networks --- including autoregressive models e.g.~\citep{sutskever2011generating,graves2013generating}, flow-based models~\citep{rezende2015variational}, deep VAEs~\citep{vahdat2020nvae} and diffusion models~\citep{sohl2015deep} --- is that the joint distribution they encode is broken down into a series of steps, thereby eluding the ``curse of dimensionality'' that would doom any effort to explicitly define all the interactions among so many variables.
In colloquial terms they solve a hard problem by splitting it into easy pieces.

A general way to view such distributions is as an exchange of messages between a sender, Alice, who has access to some data, and her friend Bob, who wishes to receive it in as few bits as possible. 
At each step Alice sends a message to Bob that reveals something about the data.
Bob attempts to guess what the message is: the better his guess the fewer bits are needed to transmit it.
After receiving the message, Bob uses the information he has just gained to improve his guess for the next message.
The loss function is the total number of bits required for all the messages.

In an autoregressive language model, for example, the messages are the word-pieces the text is divided into. 
The distribution encoding Bob’s prediction for the first message is of necessity uninformed: a zero-gram prior based on the relative frequencies of different word-pieces. 
The transmission cost is the negative log-probability under this prior. 
Bob then uses the first word-piece to predict the second; on average, the second prediction will be slightly more informed than the first, and the expected transmission cost will be slightly lower. 
The process repeats with the predictions improving at each step. 
The sum of the transmission costs is the negative log-probability of the complete text sequence, which is the loss function minimised by maximum likelihood training.  
It is also the minimum number of bits that would be required for Alice to transmit the pieces to Bob using arithmetic coding~\citep{witten1987arithmetic}. 
There is therefore a direct correspondence between fitting an autoregressive model with maximum likelihood and training it for data compression.

Autoregressive networks are currently state-of-the-art for language modelling~\citep{openai2023gpt4}, and in general perform well on discrete data where a natural ordering exists.
However they have proved less effective in domains such as image generation, where the data is continuous and no natural order exists among variables (e.g. there is no reason to generate one pixel before another).
They also have the drawback that generating samples requires as many network updates as there are variables in the data.

Diffusion models are an alternative framework that has proved particularly effective for image generation~\cite{dhariwal2021diffusion,rombach2022high}.
In this case the transmission procedure is a little more complex\footnote{We are here describing the reverse process of diffusion models.}.
Each message Bob receives is a noisy version of the message before, where the noise is designed so that in expectation the messages approach the data.
The transmission cost at each step is the Kullback-Leibler divergence between the distribution from which Alice draws the message and Bob's prediction of that distribution (which is a reparameterisation of his prediction of the data, and which is therefore improved by the information he gained from the previous message).
The sum of the KL divergences is the \emph{evidence lower bound} minimised by diffusion training~\citep{sohl2015deep}; it is also the expected number of bits needed to transmit the data using an efficient bits-back coding scheme~\citep{Wallace1991ClassificationBM,hinton1993keeping}. 
Once again there is an exact equivalence between the loss function used to train the model and the model’s ability to compress data, as elucidated by previous authors~\citep{townsend2019practical}.

We posit that the superiority of diffusion over autoregression for image generation lies in the way diffusion progresses from coarse to fine image details as the level of noise decreases --- a more natural way to construct an image than one dot at a time.
However diffusion has yet to match autoregression for discrete data, which is unfortunate, as diffusion models have the advantage of decoupling the number of generation steps from the number of variables.
A fundamental challenge is that when the data is discrete, the noise in the diffusion process is also discrete, and therefore discontinuous.
To return to the transmission metaphor, if the data is a piece of text, then Bob begins the process with a totally garbled text, every symbol of which is either randomly altered or left unchanged by each of Alice's messages.
A key motivation for this work was our belief that a fully continuous transmission process --- where Alice's messages smoothly alter Bob's beliefs --- would be more effective for discrete data.
Moreover this should open the door to gradient-based sample guidance~\citep{dhariwal2021diffusion} and few-step generation techniques~\citep{salimans2022progressive,watson2022learning,song2023consistency}, similar to those that have been developed for continuous diffusion.

\begin{figure}[t!] 
    \includegraphics[width=\textwidth]{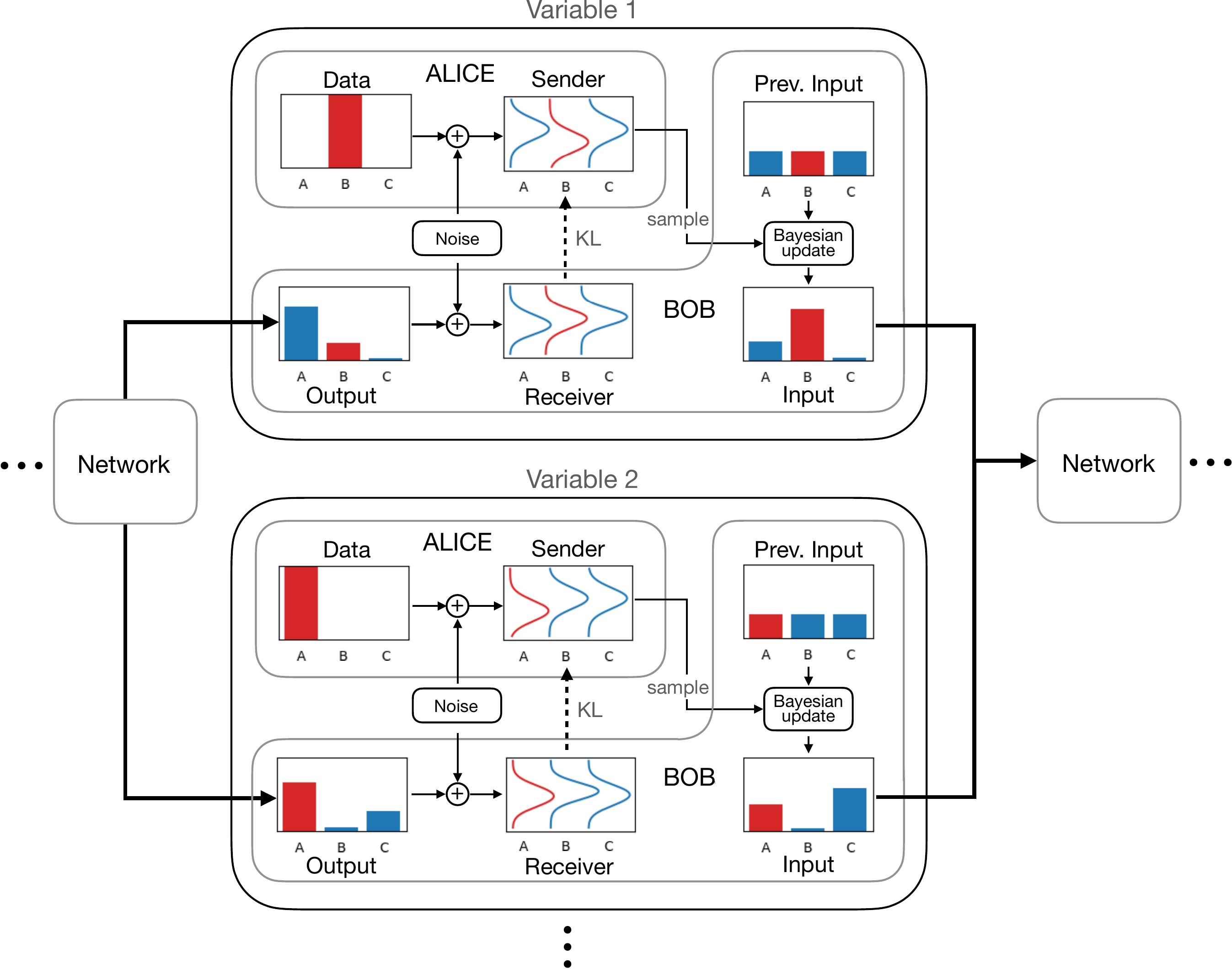}
\caption{\textbf{System Overview}. The figure represents one step of the modelling process of a Bayesian Flow Network. The data in this example is a ternary symbol sequence, of which the first two variables (`B' and `A') are shown.  At each step the network emits the parameters of the output distribution based on the parameters of the previous input distribution.  The sender and receiver distributions (both of which are continuous, even when the data is discrete) are created by adding random noise to the data and the output distribution respectively.  A sample from the sender distribution is then used to update the parameters of the input distribution, following the rules of Bayesian inference. Conceptually, this is the message sent by Alice to Bob, and its contribution to the loss function is the KL divergence from the receiver to the sender distribution.}
\label{fig:overview}
\end{figure}

\emph{Bayesian Flow Networks} (BFNs), the model introduced in this paper, differ from diffusion models in that the network operates on the parameters of a data distribution, rather than on a noisy version of the data itself.
This ensures that the generative process is fully continuous and differentiable, even when the data is discrete.
BFNs can be summarised by the following transmission scheme (Figure~\ref{fig:overview}).
Bob has an ``input distribution'' which is initially a simple prior: a standard normal for continuous data, a uniform categorical for discrete data.
At each transmission step he feeds the parameters of the input distribution (e.g. the mean of a normal distribution, the probabilities of a categorical distribution) into a neural network.
The network outputs the parameters of a second distribution referred to as the ``output distribution''.
Alice then creates a ``sender distribution'' by adding noise to the data according to a predefined schedule, and Bob creates a ``receiver distribution'' by convolving the output distribution with the same noise distribution used by Alice: intuitively, for every value the data could take on, Bob constructs the sender distribution Alice would have used if that value was correct, then sums over all these hypothetical sender distributions, weighted by the probability of the corresponding value under the output distribution.
Alice picks a sample from the sender distribution and sends it to Bob at a cost equal to the KL divergence from receiver to sender.
Bob then uses the sample to update his input distribution, following the rules of Bayesian inference.
Usefully, the Bayesian updates are available in closed-form as long as the input distribution models all the variables in the data independently.
Once the update is complete, Bob again feeds the parameters of the input distribution to the network which returns the parameters of the output distribution.
The process repeats for $n$ steps, at which point Bob can predict the data accurately enough that Alice can send it to him without any noise.

Note the key difference between the input and output distributions: the input distribution receives information about each variable in the data independently (via the Bayesian updates), and is therefore unable to exploit contextual information, such as neighbouring pixels in an image or related words in a text; the output distribution, on the other hand, is produced by a neural network that jointly processes all the parameters in the input distribution, giving it access to all available context.
Intuitively, the combination of the input and output distributions represents a division of labour between Bayesian inference and deep learning that plays to both of their strengths: the former provides a mathematically optimal and finely controllable way to collect and summarise information about individual variables, while the latter excels at integrating information over many interrelated variables.

The above transmission process defines an $n$-step loss function that can be generalised to continuous time by sending $n$ to $\infty$.
In continuous time the Bayesian updates become a \emph{Bayesian flow} of information from the data to the network.
As well as removing the need to predefine the number of steps during training, the continuous-time loss function is mathematically simpler and easier to compute than the discrete-time loss.
A BFN trained with continuous-time loss can be run for any number of discrete steps during inference and sampling, with performance improving as the number of steps increases.

The rest of the paper is structured as follows. 
A short summary of related work is given in Section~\ref{sec:related}.
The basic framework of BFNs, along with a general derivation of the discrete and continuous time loss functions is provided in Section~\ref{sec:bfn}.
Specialisations of the framework to continuous, discretised and discrete data are provided in Sections~\ref{sec:cts}--\ref{sec:discrete}, along with pseudocode for training, evaluating and sampling from the network.
Experimental results on the CIFAR-10, dynamically binarized MNIST and text8 datasets are provided in Section~\ref{sec:experiments} and concluding remarks are given in Section~\ref{sec:conclusion}.

\section{Related Work}\label{sec:related}
Of existing methods, Bayesian Flow Networks are most closely related to diffusion models.
However the two differ in some crucial aspects.
Most obviously BFNs embody a function from one distribution to another --- rather than from data to a distribution, like diffusion models and most other probabilistic networks. 
One advantage of this approach is that, because the parameters of a categorical distribution are real-valued probabilities, the inputs to the network are continuous even when the data is discrete. 
This contrasts with discrete diffusion, which natively uses discrete samples as input~\citep{sohl2015deep,hoogeboom2021,austin2021d3pm}.

Numerous authors have proposed continuous variants of discrete diffusion.
Typically these rely either on mapping to and from a continuous embedding space~\citep{strudel2022self,li2022diffusionlm,dieleman2022continuous,chen2022analog}, or on restricting continuous diffusion to the probability simplex~\citep{richemond2022categorical,mahabadi2023tess,lou2023reflected,han2023ssd}.
While we do not directly compare against the above methods, we note that continuity is an inherent property of the Bayesian Flow framework (the network inputs automatically lie on the probability simplex by virtue of being the parameters of a categorical distribution), rather than a constraint added to an existing system.
As well as reducing the number of free parameters and design choices (e.g. the continuous embedding space, the mapping functions), this ensures that BFNs directly optimise the negative log-likelihood of discrete data, unlike continuous diffusion methods for discrete data, which typically require either simplified loss functions~\citep{mahabadi2023tess} or auxiliary loss terms~\citep{li2022diffusionlm} to make learning stable.

For continuous data, BFNs are most closely related to variational diffusion models~\citep{kingma2021variational}, with a very similar continuous-time loss function.
The main difference in this case is that the network inputs are considerably less noisy in BFNs than in variational diffusion and other continuous diffusion models. 
This is because the generative process of BFNs begins with the parameters of a fixed prior, whereas that of diffusion models begins with pure noise.
We hypothesise that the reduction in noise could lead to faster learning on large datasets where the model underfits; however we have yet to test this hypothesis experimentally.

Another key difference from diffusion models is that there is no need to define and invert a forward process for BFNs, which arguably makes it easier to adapt them to different distributions and data types.
We showcase this flexibility by adapting BFNs to continuous, discretised and discrete data, with minimal changes to the training procedure.
This contrasts with e.g.\ discretised diffusion, which requires carefully defined transition matrices~\citep{austin2021d3pm}.
%%%%%%%%%%%%%%%%%%%%%%%%%%%%%%%%%%%%%%%%%%%%%%%%%%%%
\section{Bayesian Flow Networks}\label{sec:bfn}
This section covers the basic mathematical formalism of Bayesian Flow Networks, laying out the structure of the various functions and distributions required by the model, along with the discrete and continuous-time loss functions used for training.
Specific instantiations of the general framework for continuous, discretised and discrete data are given in Sections~\ref{sec:cts}--\ref{sec:discrete}. 
%%%%%%%%%%%%%%%%%%%%%%%%%%%%%%%%%%%%%%%%%%%%%%%%%%%%
\subsection{Input and Sender Distributions}
Given $D$-dimensional data $\x = \left(\didx{x}{1},\dots,\didx{x}{D}\right) \in \X^D$, let $\parsn = \left(\parsdd{1},\dots,\parsdd{D}\right)$ be the parameters of a factorised \emph{input distribution} $\inp(\cdot \mid \parsn)$, with
\begin{align}
\inp(\x \mid \parsn) = \prod_{d=1}^D \inp(\didx{x}{d} \mid \parsdd{d}).
\end{align}
For example, $\parsdd{d}$ may consist of the probabilities of a categorical distribution. 
Let $\sender{\cdot}{\x;\alpha}$ be a similarly factorised \emph{sender distribution} with $\y =\left(\didx{y}{1},\dots,\didx{y}{D}\right) \in \Y^D$ and
\begin{align}
\sender{\y}{\x;\alpha} = \prod_{d=1}^D \sender{\didx{y}{d}}{\didx{x}{d}; \alpha},
\end{align}
where $\alpha \in \R^+$ is an \emph{accuracy} parameter defined such that when $\alpha=0$, the sender samples are entirely uninformative about $\x$ and as $\alpha$ increases the samples become progressively more informative.
%%%%%%%%%%%%%%%%%%%%%%%%%%%%%%%%%%%%%%%%%%%%%%%%%%%%
\subsection{Output Distribution \texorpdfstring{$\out(\cdot \mid \parsn, t)$}{}}
During the data transmission process, the input parameters $\parsn$ are passed along with the process time $t$ as input to a neural network $\net$.
The network then emits an output vector $\net(\parsn, t) = \left(\didx{\net}{1}(\parsn, t),\dots,\didx{\net}{D}(\parsn, t)\right)$ which is used to parameterise an \textit{output distribution} factorised in the same way as the input and sender distributions:
\begin{align}
\out(\x \mid \parsn, t) = \prod_{d=1}^D \out(\didx{x}{d} \mid \didx{\net}{d}(\parsn, t)).
\end{align}
As discussed in the introduction, the key difference between the input and output distributions is that while each $\inp(\didx{x}{d} \mid \parsdd{d})$ depends only on information gathered via $\sender{\didx{y}{d}}{\didx{x}{d};\alpha}$ about $\didx{x}{d}$, each $\out(\didx{x}{d} \mid \didx{\net}{d}(\parsn, t))$ depends (via the network) on all of $\parsn$ and hence all of $\x$.
The output distribution, unlike the input distribution, can therefore exploit context information, such as surrounding pixels in an image or related words in a text.
%%%%%%%%%%%%%%%%%%%%%%%%%%%%%%%%%%%%%%%%%%%%%%%%%%%%
\subsection{Receiver Distribution \texorpdfstring{$\rec(\cdot \mid \parsn; t, \alpha)$}{}}
Given sender distribution $\sender{\cdot}{\x; \alpha}$ and output distribution $\out(\cdot \mid \parsn, t)$ the \emph{receiver distribution} over $\Y^D$ is defined as
\begin{align}
\rec(\y \mid \parsn; t, \alpha) &= \E_{\out(\x' \mid \parsn; t)}\sender{\y}{\x'; \alpha}.\label{r_dist}
\end{align}
Intuitively this can be understood as a receiver who knows the form of the sender distribution $\sender{\cdot}{\x ; \alpha}$ but does not know $\x$, and therefore integrates over all $\x' \in \X^D$, and hence all possible sender distributions, weighted by the probability  given to $\x'$ by the output distribution $\out(\x \mid \parsn, t)$.
The receiver distribution therefore combines two sources of uncertainty: the ``known unknown'' of the sender distribution entropy (which is a function of $\alpha$), and the ``unknown unknown'' of the output distribution entropy.
%%%%%%%%%%%%%%%%%%%%%%%%%%%%%%%%%%%%%%%%%%
\subsection{Bayesian Updates}
Given parameters $\parsn$ and sender sample $\y$ drawn with accuracy $\alpha$ the \emph{Bayesian update function} $h$ is derived by applying the rules of Bayesian inference to compute the updated parameters $\parsn'$: 
\begin{align}
\parsn' \leftarrow h(\parsn, \y, \alpha).
\end{align} 
The \emph{Bayesian update distribution} $\update(\cdot \mid \parsn, \x; \alpha)$ is then defined by marginalizing out $\y$:
\begin{align}
\update(\parsn' \mid \parsn, \x; \alpha) = \E_{\sender{\y}{\x;\alpha}} \delta \left(\parsn' -h(\parsn, \y, \alpha) \right),\label{param_update_dist}
\end{align}
where $\delta \left(\cdot -\vec{a}\right)$ is the multivariate Dirac delta distribution centred on the vector $\vec{a}$.
In Sections~\ref{sec:cts_additive} and \ref{sec:disc_additive} we will prove that both forms of $\update(\cdot \mid \parsn, \x; \alpha)$ considered in this paper have the following property: the accuracies are additive in the sense that if $\alpha = \alpha_a + \alpha_b$ then
\begin{align}
\update(\parsn'' \mid \parsn, \x; \alpha) = \E_{\update(\parsn' \mid \parsn, \x; \alpha_a)} \update(\parsn'' \mid \parsn', \x; \alpha_b)\label{additive}.
\end{align}
It follows from this property that given prior input parameters $\parsnt{0}$, the probability of observing parameters $\parsnt{n}$ after drawing a sequence of $n$ sender samples $\yt{1},\dots,\yt{n}$ with accuracies $\alpha_1,\dots,\alpha_n$ is
\begin{align}
\E_{\update(\parsnt{1}\mid\parsnt{0},\x;\alphat{1})}\E_{\update(\parsnt{2}\mid\parsnt{1},\x;\alphat{2})}\dots\E_{\update(\parsnt{n-1}\mid\parsnt{n-2},\x;\alphat{n-1})}\update(\parsnt{n} \mid \parsnt{n-1},\x;\alphat{n} ) = \update\left(\parsnt{n} \mid \parsnt{0}, \x; \sum_{i=1}^n \alphat{i}\right).
\label{updateseq}
\end{align}
%%%%%%%%%%%%%%%%%%%%%%%%%%%%%%%%%%%%%%%%%%%%%%%%%%%%
\subsection{Accuracy Schedule \texorpdfstring{$\beta(t)$}{}}
By performing an infinite number of transmission steps, the Bayesian update process can be generalized to continuous time.
Let $t \in [0,1]$ be the process \textit{time} and let $\alpha(t) > 0$ be the \emph{accuracy rate} at time $t$.
Now define the \emph{accuracy schedule} $\beta(t)$ as 
\begin{align}
\beta(t) = \int_{t'=0}^{t} \alpha(t') dt'.
\end{align}
It follows from the above definitions that $\beta(t)$ is a monotonically increasing function of $t$, that $\beta(0) = 0$, and that $\frac{d\beta(t)}{dt} = \alpha(t)$.

Specific forms of $\beta(t)$ for continuous and discrete data are provided in Sections~\ref{sec:cts_beta} and \ref{sec:disc_beta}. 
Both are derived using simple heuristics, with a deeper investigation left for future work.
%%%%%%%%%%%%%%%%%%%%%%%%%%%%%%%%%%%%%%%%%%%%%%%%%%%%
\subsection{Bayesian Flow Distribution \texorpdfstring{$\flow(\cdot \mid \x ; t)$}{}}
Given prior parameters $\parsnt{0}$, Bayesian update distribution $\update(\cdot \mid \parsn, \x; \alpha)$ and accuracy schedule $\beta(t)$, the \emph{Bayesian flow distribution} $\flow(\cdot \mid \x ; t)$ is the marginal distribution over input parameters at time $t$, defined by
\begin{align}
\flow(\parsn \mid \x ; t) = \update(\parsn \mid \parsnt{0}, \x; \beta(t))\label{param_flow_dist}.
\end{align}
%%%%%%%%%%%%%%%%%%%%%%%%%%%%%%%%%%%%%%%%%%%%%%%%%%%%
\subsection{Loss Function \texorpdfstring{$L(\x)$}{}}
Given prior parameters $\parsnt{0}$ and accuracy schedule $\beta(t)$, consider a sequence of $n$ sender samples $\yt{1},\dots,\yt{n}$ sampled at times $t_1,\dots,t_n$ where $t_i = i/n$. The sender distribution at step $i$ is $\sender{\cdot}{\x; \alpha_i}$
where
\begin{align}
\alphat{i} &= \beta(t_i) - \beta(t_{i-1}),\label{alpha_i}
%&= \beta(i/n) - \beta((i-1)/n)
\end{align}
the receiver distribution at step $i$ is $\rec(\cdot \mid \parsnt{i-1}; t_{i-1}, \alphat{i})$, 
and the input parameter sequence $\parsnt{1},\dots,\parsnt{n}$ is recursively calculated from
\begin{align}
\parsnt{i} = h(\parsnt{i-1}, \y, \alphat{i}).
\end{align}
Define the $n$-step \textit{discrete-time loss} $L^n(\x)$ as the expected number of nats required to first transmit $\yt{1},\dots,\yt{n}$, and the \textit{reconstruction loss} $L^r(\x)$ as the expected number of nats required to then transmit $\x$.
Since --- using a bits-back coding scheme~\citep{hinton1993keeping, duda2009asymmetric} --- it requires $\kl{p_{_S}}{\rec}$ nats to transmit a sample from $p_{_S}$ to a receiver with $\rec$,
\begin{align}
L^n(\x) = \E_{p(\parsnt{1},\dots,\parsnt{n-1})}\sum_{i=1}^n \kl{\sender{\cdot}{\x ; \alphat{i}}}{\rec(\cdot \mid \parsnt{i-1} ; t_{i-1}, \alphat{i})}\label{disc_t_loss_n_step},
\end{align}
where
\begin{align}
p(\parsnt{1},\dots,\parsnt{n}) = \prod_{i=1}^{n} \update(\parsnt{i}\mid \parsnt{i-1}, \x ; \alphat{i}),
\end{align}
and since the number of nats needed to transmit $x$ using an arithmetic coding scheme~\citep{witten1987arithmetic} based on $p(x)$ is $-\ln p(x)$, and the marginal probability of $\parsnt{n}$ is given by $\flow(\cdot \mid \x, 1)$,
\begin{align}
L^r(\x) = -\E_{\flow(\parsn \mid \x, 1)} \ln \out(\x \mid \parsn; 1).
\end{align}
Note that $L^r(\x)$ is not directly optimised in this paper; however it is indirectly trained by optimising $L^n(\x)$ since both are minimised by matching the output distribution to the data.
Furthermore, as long as $\beta(1)$ is high enough, the input distribution at $t=1$ will be very close to $\x$, making it trivial for the network to fit $\out(\x \mid \parsn; 1)$.

The loss function $L(\x)$ is defined as the total number of nats required to transmit the data, which is the sum of the n-step and reconstruction losses:
\begin{align}
L(\x) = L^n(\x) + L^r(\x)
\end{align}
Alternatively  $L(\x)$ can be derived as the loss function of a variational autoencoder (VAE;~\citep{kingma2013auto}). Consider the sequence $\yt{1},\dots,\yt{n}$ as a latent code with posterior probability given by
\begin{align}
q(\yt{1},\dots,\yt{n}) = \prod_{i=1}^n \sender{\y_i}{\x; \alpha_i},
\end{align}
and autoregressive prior probability given by
\begin{align}
p(\yt{1},\dots,\yt{n}) = \prod_{i=1}^n \rec(\y_i \mid \parsnt{i-1}; t_{i-1}, \alphat{i}).
\end{align}
Then, noting that the decoder probability $p(\x \mid \yt{1},\dots,\yt{n}) = \out(\x \mid \parsnt{n}; 1)$, the complete transmission process defines a VAE with loss function given by the negative variational lower bound (VLB)
\begin{align}
L(\x) = - \text{VLB}(\x) &= \kl{q}{p} - \E_{\yt{1},\dots,\yt{n} \sim q} \ln p(\x \mid \yt{1},\dots,\yt{n})\label{vae_loss}\\
&=L^n(\x) + L^r(\x).
\end{align}
%%%%%%%%%%%%%%%%%%%%%%%%%%%%%%%%%%%%%%%%%%%%%%%%%%%%
\subsection{Discrete-Time Loss \texorpdfstring{$L^{n}(\x)$}{}}
Eq.~\ref{disc_t_loss_n_step} can be rewritten as
\begin{align}
L^{n}(\x) = n\E_{i \sim \ui{n}} \E_{\update(\parsnt{1} \mid \parsnt{0}, \x ; \alphat{1})}\dots\E_{\update(\parsn \mid \parsnt{i-2}, \x ; \alphat{i-1})} \kl{\sender{\cdot}{\x ; \alphat{i}}}{\rec(\cdot \mid \parsn ; t_{i-1}, \alphat{i})},
\end{align}
where $\ui{n}$ is the uniform distribution over the integers from 1 to $n$.
Furthermore, it follows from Eqs.~\ref{updateseq} and ~\ref{param_flow_dist} that
\begin{align}
\E_{\update(\parsnt{1} \mid \parsnt{0}, \x ; \alphat{1})}\dots\E_{\update(\parsn \mid \parsnt{i-2}, \x ; \alphat{i-1})} &= \E_{\update(\parsn \mid \parsnt{0}, \x ; \beta(t_{i-1}))}\\
&= \E_{\flow(\parsn \mid \x ; t_{i-1})},
\end{align}
and hence
\begin{align}
L^{n}(\x) = n \E_{i \sim \ui{n}, \flow(\parsn \mid \x ; t_{i-1})} \kl{\sender{\cdot}{\x ; \alphat{i}}}{\rec(\cdot \mid \parsn; t_{i-1}, \alphat{i})}\label{disc_t_loss_exp},
\end{align}
which allows us approximate $L^{n}(\x)$ via Monte-Carlo sampling without computing the $n$-step sum.
%%%%%%%%%%%%%%%%%%%%%%%%%%%%%%%%%%%%%%%%%%%%%%%%%%%%
\subsection{Continuous-Time Loss \texorpdfstring{$L^{\infty}(\x)$}{}}
Eq.~\ref{disc_t_loss_exp} can be used to train the network directly.
However this presupposes that $n$ is fixed during training.
Furthermore, for discrete and discretised data the KL terms do not have analytic solutions, leading to noisy gradient estimates.

Inspired by Variational Diffusion Models~\cite{kingma2021variational} we derive a continuous-time loss function $L^{\infty}(\x)$ by taking the limit of $L^{n}(\x)$ as $n \rightarrow \infty$.
This turns out to be mathematically simpler than the discrete-time loss, as well as removing both the noisy gradients for the discrete and discretised KL terms and the need to fix $n$ during training.

Let
\begin{align}
\epsilon &\defeq \frac{1}{n},\\
\alpha(t, \epsilon) &\defeq \beta(t) - \beta(t-\epsilon),\label{deltat}\\
L^{\infty}(\x) &\defeq \lim_{n\rightarrow\infty}L^n(\x).
\end{align}
Then, from the definition of $L^n(\x)$ in Eq.~\ref{disc_t_loss_exp},
\begin{align}
L^{\infty}(\x) = \lim_{\epsilon \rightarrow 0} \frac{1}{\epsilon} \E_{t \sim U(\epsilon,1), \flow(\parsn \mid \x, t-\epsilon)} \kl{\sender{\cdot}{\x; \alpha(t, \epsilon)}}{\rec(\cdot \mid \parsn; t-\epsilon, \alpha(t, \epsilon))},
\end{align}
where $U(a,b)$ is the continuous uniform distribution over the interval $[a,b]$.
As we will see, for all the sender, receiver distribution pairs in this paper,
\begin{align}
\kl{\sender{\cdot}{\x; \alpha}}{\rec(\cdot \mid \parsn; \alpha, t)} = \sum_{d=1}^D\kl{\N{g(\xdd{d})}{C\alpha^{-1}}}{P^{(d)}(\parsn, t) \ast \N{0}{C\alpha^{-1}}}\label{convkl},
\end{align}
where $g: \X \rightarrow \Y$ is a function from data space to sender space, $P^{(d)}(\parsn, t)$ is a distribution over $\Y$ with finite expectation and variance, $\ast$ denotes the convolution of two probability distributions and $C$ is a scalar constant.

The following proposition is now required:
\begin{proposition}\label{proposition}
For a continuous univariate probability distribution $P$ with finite expectation $E[P]$ and variance $Var[P]$, the convolution $P \ast \N{0}{\sigma^2} \rightarrow \N{E[P]}{\sigma^2}$ as $\sigma^2 \rightarrow \infty$.
\end{proposition}

\begin{proof}
Let $\epsilon^2$ be some variance in the interval $\left(0, \frac{\pi}{8}\right)$ and consider the sequence of random variables $X_0,X_1,\dots,X_n$ where $X_0 \sim P$ and $X_j \sim \N{0}{\epsilon^2}$ for $j > 0$. Define
\begin{align}
Y_j &\defeq \begin{cases}X_0 - E[P]&\text{if } j=0,\\ X_j &\text{ otherwise.}\end{cases}\\
R_n &\defeq \sum_{j=0}^n Y_j,\\
S^2_n &\defeq \sum_{j=1}^n Var[Y_j] = n \epsilon^2,\\
T^2_n &\defeq Var[P] + S^2_n.
\end{align}
It follows from the definition of convolution that $\sum_{j=0}^n X_j \sim P \ast \N{0}{n\epsilon^2}$. 
Since $n \epsilon^2 \rightarrow \infty$ as $n \rightarrow \infty$, and $\sum_{j=0}^n X_j =  R_n + E[P]$, the result is proved if it can be shown that as $n \rightarrow \infty$, $R_n \rightarrow \N{0}{n\epsilon^2}$ or equivalently $R_n/(\epsilon\sqrt{n}) \rightarrow \N{0}{1}$.

\sloppy The Lyapunov central limit theorem~\citep{georgii2008stochastics} states that if there exists $\lambda > 0$ such that $\lim_{n\rightarrow \infty}\frac{1}{T_n^{2+\lambda}}\sum_{j=0}^n E\left(|Y_j|^{2+\lambda}\right) = 0$ then $R_n/T_n \rightarrow \N{0}{1}$.
First note that $T_n^2 \rightarrow S_n^2 = n\epsilon^2$ as $n \rightarrow \infty$.
Hence if $R_n/T_n \rightarrow \N{0}{1}$ then $R_n/(\epsilon\sqrt{n}) \rightarrow \N{0}{1}$.
Now set $\lambda=1$ and observe that for $Y_j \sim \N{0}{\epsilon^2}$, $\E\left(|Y_j|^{3}\right)$ is the third moment of the half-normal distribution, which is $\epsilon^3\sqrt{\frac{8}{\pi}}$.
Our choice of $\epsilon^2$ therefore ensures that $E\left(|Y_j|^{3}\right) < \epsilon^2$ for $j > 0$.
Also note that $T_n^3 > S_n^3$ and, since $E[P]$ and $Var[P]$ are finite, $E\left(|Y_0|^{3}\right) < C$ for some constant $C$.
Hence
\begin{align}
\frac{1}{T_n^3}\sum_{j=0}^n E\left(|Y_j|^{3}\right) &< 
\frac{1}{S_n^{3}}\left(C + n\epsilon^2\right) = \frac{C}{\epsilon^3 n^{3/2}} + \frac{1}{\epsilon\sqrt{n}} \xrightarrow[]{n\rightarrow\infty}0.
\end{align}
\end{proof}
It follows from the continuity of $\beta(t)$ and Eq.~\ref{deltat} that $\alpha(t, \epsilon)^{-1} \rightarrow \infty$ as $\epsilon \rightarrow 0$.
Therefore, Proposition \ref{proposition} can be applied to Eq.~\ref{convkl} to yield
\begin{align}
\lim_{\epsilon \rightarrow 0} \kl{\sender{\cdot}{\x, \alphat{t}}}{\rec(\cdot \mid \parsn, \alphat{t}, t)} &= \sum_{d=1}^D\kl{\N{g(\xdd{d})}{\frac{C}{\alpha(t, \epsilon)}}}{\N{E[P^{(d)}(\parsn, t)]}{\frac{C}{\alpha(t, \epsilon)}}}\label{convkllim}\\
&= \frac{\alpha(t, \epsilon)}{2C} \left\|g(\x) - E[P(\parsn, t)]\right\|^2,
\end{align}
where 
\begin{align}
g(\x) = \left(g(\xdd{1}),\dots,g(\xdd{D})\right),\\
E[P(\parsn, t)] = \left(E[P^{(1)}(\parsn, t)],\dots,E[P^{(D)}(\parsn, t)]\right).
\end{align}
Therefore,
\begin{align}
L^{\infty}(\x) = \E_{t \sim U(0,1), \flow(\parsn \mid \x, t)} \lim_{\epsilon \rightarrow 0} \frac{\alpha(t, \epsilon)}{\epsilon} \frac{\left\|g(\x) -  E[P(\parsn, t)]\right\|^2}{2C}.
\end{align}
Substituting from Eq.~\ref{deltat},
\begin{align}
\lim_{\epsilon \rightarrow 0} \frac{\alpha(t, \epsilon)}{\epsilon} = \lim_{\epsilon \rightarrow 0}\frac{\beta(t)-\beta(t-\epsilon)}{\epsilon} = \frac{d \beta(t)}{d t} = \alpha(t),
\end{align}
and hence
\begin{align}
L^{\infty}(\x) &= \E_{t \sim U(0,1), \flow(\parsn \mid \x, t)} \alpha(t) \frac{\left\|g(\x) -  E[P(\parsn, t)]\right\|^2}{2C}.\label{cts_t_loss}
\end{align}
%
%%%%%%%%%%%%%%%%%%%%%%%%%%%%%%%%%%%%%%%%%%%%%%%%%%%%
\subsection{Sample Generation}
Given prior parameters $\parsnt{0}$, accuracies $\alphat{1},\dots,\alphat{n}$ and corresponding times $t_i = i/n$, the n-step sampling procedure recursively generates $\parsnt{1},\dots,\parsnt{n}$ by sampling $\x'$ from $\out(\cdot \mid \parsnt{i-1}, t_{i-1})$, 
$\y$ from $\sender{\cdot}{\x', \alphat{i}}$ (meaning that $\y \sim \rec(\cdot \mid \parsnt{i-1}; t_{i-1}, \alphat{i})$ --- see Eq.~\ref{r_dist}), then setting 
$\parsnt{i} = h(\parsnt{i-1}, \y)$.
Given $\parsnt{n}$ the network is run one more time and the final sample is drawn from $ \out(\cdot \mid \parsnt{n}, 1)$.
%%%%%%%%%%%%%%%%%%%%%%%%%%%%%%%%%%%%%%%%%%%%%%%%%%%%
\section{Continuous Data}\label{sec:cts}
For continuous data $\X = \R$ and hence $\x \in \R^D$.
In our experiments, $\x$ is normalised to lie in $[-1, 1]^D$ to ensure that the network inputs remain in a reasonable range; however this is not essential for the mathematical framework.
%%%%%%%%%%%%%%%%%%%%%%%%%%%%%%%%%%%%%%%%%%%%%%%%%%%%
\subsection{Input Distribution \texorpdfstring{$\inp(\cdot \mid \parsn)$}{}}\label{sec:cts_input}
The input distribution for continuous data is a diagonal normal:
\begin{align}
\parsn &\defeq \{\m, \rho\}\\
\inp(\x \mid \parsn) &\defeq \N{\x \mid \m}{\rho^{-1}\I{D}},
\end{align}
where $\I{D}$ is the $D \times D$ identity matrix. 
We define the prior parameters as
\begin{align}
\parsnt{0} \defeq \{\0{D}, 1\},
\end{align}
where $\0{D}$ is the length $D$ vectors of zeros.
Hence the input prior is a standard multivariate normal:
\begin{equation}
\inp(\x \mid \parsnt{0}) = \N{\x \mid \0{D}}{\I{D}}.
\end{equation}
The usual Bayesian approach would be to fit the prior mean and variance to the training data. 
However we found that a standard prior worked better in practice, as well as simplifying the equations. 
It is important to remember that the distributions $\inp(\x \mid \parsnt{0})$ are never used directly to make predictions, but rather to inform the network's predictions. 
All that matters is that the parameters fed into the network accurately and accessibly encode the information received so far about $\x$. 
The network can easily learn the empirical prior of the training set and use that to correct its predictions.
%%%%%%%%%%%%%%%%%%%%%%%%%%%%%%%%%%%%%%%%%%%%%%%%%%%%
\subsection{Bayesian Update Function \texorpdfstring{$h(\parsnt{i-1}, \y, \alpha)$}{}}
Given a univariate Gaussian prior $\N{\mu_a}{\pt{a}^{-1}}$ over some unknown data $x$ it can be shown~\citep{murphy2007conjugate} that the Bayesian posterior after observing a noisy sample $y$ from a normal distribution $\N{x}{\alpha^{-1}}$ with known precision $\alpha$ is $\N{\mu_b}{\pt{b}^{-1}}$, where
\begin{align}
\pt{b} &= \pt{a} + \alpha\label{alpha_update},\\
\mu_b &= \frac{\mu_a \pt{a}  + y \alpha}{\pt{b}}\label{mean_update}.
\end{align}
Since both $\inp(\x \mid \parsn)$ and $\sender{\y}{\x; \alpha}$ distributions are normal with diagonal covariance,  Eqs.~\ref{alpha_update} and \ref{mean_update} can be applied to obtain the following Bayesian update function for parameters $\parsnt{i-1} = \{\mt{i-1}, \pt{i-1}\}$ and sender sample $\y$ drawn from $\sender{\cdot}{\x; \alpha \I{D}} = \N{\x}{\alpha^{-1}\I{D}}$:
\begin{align}
h(\{\mt{i-1}, \pt{i-1}\}, \y, \alpha) = \{\mt{i}, \pt{i}\},
\end{align}
with
\begin{align}
\pt{i} &= \pt{i-1} + \alpha\label{cts_precision_y_update},\\
\mt{i} &= \frac{\mt{i-1} \pt{i-1} + \y \alpha}{\pt{i}}.\label{cts_mean_y_update}
\end{align}
\begin{figure}[t!]
\includegraphics[width=\textwidth]{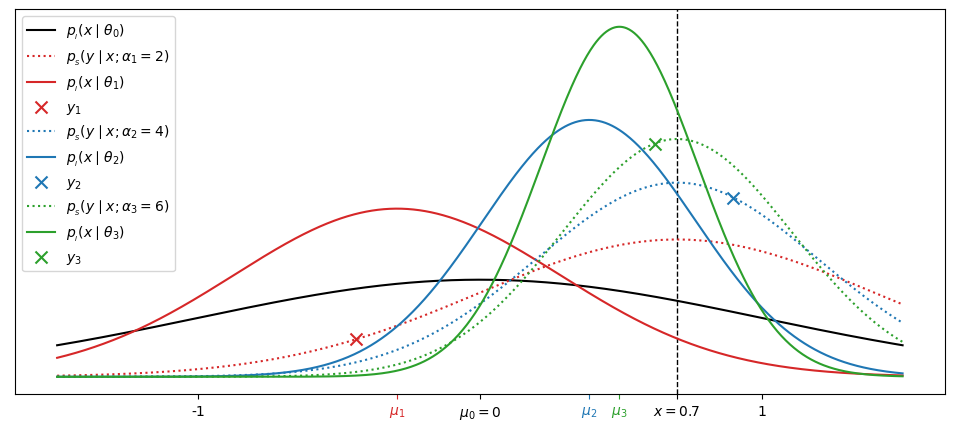}
\caption{\textbf{Bayesian updates for continuous data}. For univariate data $x=0.7$, the initial input distribution parameters $\theta_0 = \{\mu_0=0, \rho_0=1\}$ are updated to $\theta_1=\{\mu_1, \rho_1\}$, $\theta_2=\{\mu_2, \rho_2\}$, $\theta_3=\{\mu_3, \rho_3\}$ by iterating Eqs.~\ref{cts_precision_y_update} and \ref{cts_mean_y_update} with sender samples $y_1$, $y_2$, $y_3$ drawn with accuracies $2$, $4$, $6$ respectively. Note how the input mean ($\mu_1$, $\mu_2$, $\mu_3$) stochastically approaches the data, while the input precision smoothly increases.}
\end{figure}
%%%%%%%%%%%%%%%%%%%%%%%%%%%%%%%%%%%%%%%%%%%%%%%%%%%%
\subsection{Bayesian Update Distribution \texorpdfstring{$\update(\cdot \mid \parsn, \x; \alpha)$}{}}
Eq.~\ref{cts_mean_y_update} computes $\mt{i}$ given a single sample $\y$ from the sender distribution.
To marginalise over $\y \sim \N{\y \mid \x}{\alpha^{-1}\I{D}}$ as defined in Eq.~\ref{param_update_dist}, the following standard identity for normal distributions can be applied:
\begin{align}
X \sim \N{\mu_X}{\sigma_X^2} \implies aX + b \sim \N{a\mu_X + b}{a^2\sigma_X^2}\  \forall a, b \in \R.\label{normal_identity_1}
\end{align}
Substituting $X=\y$, $\mu_X=\x$, $\sigma^2_X=\alpha^{-1}\I{D}$, $a=\frac{\alpha}{\pt{i}}$ and $b=\frac{\mt{i-1}\pt{i-1}}{\pt{i}}$, Eq.~\ref{cts_mean_y_update} gives:
\begin{align}
\mt{i} \sim \N{\frac{\alpha \x + \mt{i-1}\pt{i-1}}{\pt{i}}}{\frac{\alpha}{\pt{i
}^2}\I{D}},\label{cts_input_mean_distribution}
\end{align}
and therefore (since $\mt{i}$ is the only random part of $\parsnt{i}$)
\begin{align}
\update(\parsnt{i} \mid \parsnt{i-1}, \x; \alpha) = \N{\mt{i} \mid \frac{\alpha \x + \mt{i-1}\pt{i-1}}{\pt{i}}}{\frac{\alpha}{\pt{i
}^2}\I{D}}.\label{cts_update_dist}
\end{align}
\begin{figure}[t]
\includegraphics[width=\textwidth]{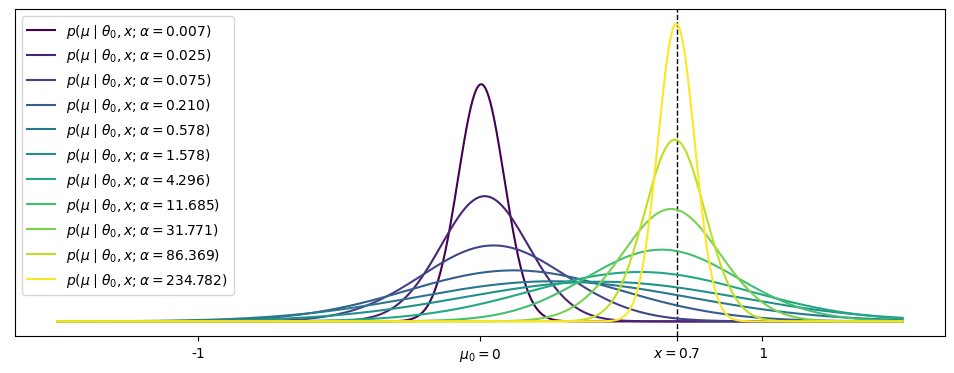}
\caption{\textbf{Bayesian update distribution for continuous data}. For $x=0.7$, the plot shows the distribution $p(\mu \mid \theta_0, x; \alpha)$ over input mean $\mu$ from Eq.~\ref{cts_input_mean_distribution} given initial parameters $\mu_0=0, \rho_0=1$ and 11 $\alpha$ values spaced log-linearly between $e^{-5}$ and $e^5$. Note how the distribution is tightly concentrated around $\mu_0$ for very low alpha, then smoothly progresses to a tight concentration around $x$ for high alpha.}
\end{figure}
%%%%%%%%%%%%%%%%%%%%%%%%%%%%%%%%%%%%%%%%%%%%%%%%%%%%
\subsection{Additive Accuracies}\label{sec:cts_additive}
We can check that the sender accuracies are additive in the sense required by Eq.~\ref{additive} by first observing that if $\parsnt{i-1} = \{\mt{i-1}, \pt{i-1}\}$ is drawn from $p(\cdot \mid \parsnt{i-2}, \x; \alpha_a)$ then
\begin{align}
\mt{i-1} \sim \N{\frac{\alpha_a \x + \mt{i-2}\pt{i-2}}{\pt{i-1}}}{\frac{\alpha_a}{\pt{i-1}^2}\I{D}}.
\end{align}
Define
\begin{align}
\boldsymbol{\mu}'_i \defeq \frac{\alpha_b \x + \mt{i-1}\pt{i-1}}{\pt{i}} = \frac{\pt{i-1}}{\pt{i}}\mt{i-1} + \frac{\alpha_b \x}{\pt{i}},
\end{align}
and apply Identity~\ref{normal_identity_1} with $a = \smash{\frac{\pt{i-1}}{\pt{i}}}$ and $b = \smash{\frac{\alpha_b\x}{\pt{i}}}$ to see that
\begin{align}
\boldsymbol{\mu}'_i  & \sim \N{\frac{\pt{i-1}}{\pt{i}}\frac{\alpha_a \x + \mt{i-2}\pt{i-2}}{\pt{i-1}} + \frac{\alpha_b \x}{\pt{i}}}{\frac{\pt{i-1}^2}{\pt{i}^2}\frac{\alpha_a}{\pt{i-1}^2}\I{D}}\\
&= \N{\frac{(\alpha_a + \alpha_b) \x + \mt{i-2}\pt{i-2}}{\pt{i}}}{\frac{\alpha_a}{\pt{i}^2}\I{D}}.
\end{align}
Now observe that if $\parsnt{i} = \{\mt{i}, \pt{i}\}$ is drawn from $p(\cdot \mid \parsnt{i-1}, \x; \alpha_b)$ then
\begin{align}
\mt{i} &\sim \N{\frac{\alpha_b \x + \mt{i-1}\pt{i-1}}{\pt{i}}}{\frac{\alpha_b}{\pt{i}^2}\I{D}},
\end{align}
and hence
\begin{align}
\mt{i} &\sim \boldsymbol{\mu}'_i + \vec{\epsilon},
\end{align}
where
\begin{align}
\vec{\epsilon} \sim \N{\0{D}}{\frac{\alpha_b}{\pt{i}^2}\I{D}}.
\end{align}
Another standard identity for Gaussian variables can now be applied:
\begin{align}
X \sim \N{\mu_X}{\sigma^2_X}, Y \sim \N{\mu_Y}{\sigma^2_Y} \implies X+Y \sim \N{\mu_X + \mu_Y}{\sigma^2_X+\sigma^2_Y}\label{normal_identity_2},
\end{align}
to see that 
\begin{align}
\mt{i} &\sim \N{\frac{(\alpha_a + \alpha_b) \x + \mt{i-2}\pt{i-2}}{\pt{i}}}{\frac{\alpha_a + \alpha_b}{\pt{i}^2}\I{D}},
\end{align}
and hence
\begin{align}
\E_{\update(\parsnt{i-1}\mid \parsnt{i-2}, \x; \alpha_a)} \update(\parsnt{i} \mid \parsnt{i-1}, \x; \alpha_b) =  \update(\parsnt{i} \mid \parsnt{i-2}, \x; \alpha_a + \alpha_b),
\end{align}
as required. 
%%%%%%%%%%%%%%%%%%%%%%%%%%%%%%%%%%%%%%%%%%%%%%%%%%%%
\subsection{Accuracy Schedule \texorpdfstring{$\beta(t)$}{}}\label{sec:cts_beta}
We derive $\beta(t)$ for continuous data by requiring that the expected entropy of the input distribution linearly decreases with $t$.
Intuitively, this means that information flows into the input distribution at a constant rate.
Define
\begin{align}
H(t) &\defeq \E_{\flow(\parsn \mid \x; t)} H(\inp(\cdot \mid \parsn))\\
&= \frac{D}{2} \ln \left(\frac{2\pi e}{1 + \beta(t)}\right).
\end{align}
Then if $H(t)$ linearly decreases with $t$,
\begin{align}
H(t) &= (1-t)H(0) + tH(1)\\
\implies \ln \left(\frac{2\pi}{1 + \beta(t)}\right) &= (1-t)\ln (2 \pi) + t \ln \left(\frac{2\pi}{1 + \beta(1)}\right)\\
\implies -\ln (1+\beta(t)) &= -t\ln(1+\beta(1))\\
\implies (1+\beta(t))^{-1} &= (1+\beta(1))^{-t}.\label{pvs}
\end{align}
Define $\sigma_1$ to be the standard deviation of the input distribution at $t=1$. 
We will choose $\sigma_1$ empirically to minimise the loss; in general it should be small enough to ensure that the reconstruction loss is low, but not so small as to create unnecessary transmission costs.
Recalling that the precision $\rho$ at time $t$ is $1+\beta(t)$, we see that
\begin{align}
\sigma_1^2 = (1 + \beta(1))^{-1}.
\end{align}
Therefore
\begin{align}
(1+\beta(t))^{-1} &= \sigma_1^{2t}\\
\implies \beta(t) &= \sigma_1^{-2t} - 1\label{cts_beta_t}\\
\implies \alpha(t) &= \frac{d \left(\sigma_1^{-2t} - 1\right)}{dt}\\
&= -\frac{2 \ln \sigma_1}{\sigma_1^{2t}}\label{ctsalphat}.
\end{align}
%%%%%%%%%%%%%%%%%%%%%%%%%%%%%%%%%%%%%%%%%%%%%%%%%%%%
\subsection{Bayesian Flow Distribution \texorpdfstring{$\flow(\cdot \mid \x; t)$}{}}
Recall from Eq.~\ref{param_flow_dist} that
\begin{align}
\flow(\parsn \mid \x; t) &= \update(\parsn \mid \parsnt{0}, \x, \beta(t)).
\end{align}
Therefore, setting $\parsnt{i-1} = \parsnt{0} = \{\0{D},1\}$ and $\alpha = \beta(t)$ in Eq.~\ref{cts_update_dist}, and recalling that $\rho = 1 + \beta(t)$,
\begin{align}
\flow(\parsn \mid \x; t) &=  \N{\m \mid \frac{\beta(t)}{1+\beta(t)}\x}{\frac{\beta(t)}{(1+\beta(t))^2}\I{D}}\\ 
&= \N{\m \mid \gamma(t)\x}{\gamma(t)(1-\gamma(t))\I{D}},\label{cts_param_flow_dist}
\end{align}
where
\begin{align}
\gamma(t) &\defeq \frac{\beta(t)}{1+\beta(t)}\label{gamma}\\
&= \frac{\sigma_1^{-2t} - 1}{\sigma_1^{-2t}}\\
&= 1 - \sigma_1^{2t}\label{cts_gamma_t}.
\end{align}
\begin{figure}[t!]
\includegraphics[width=\textwidth]{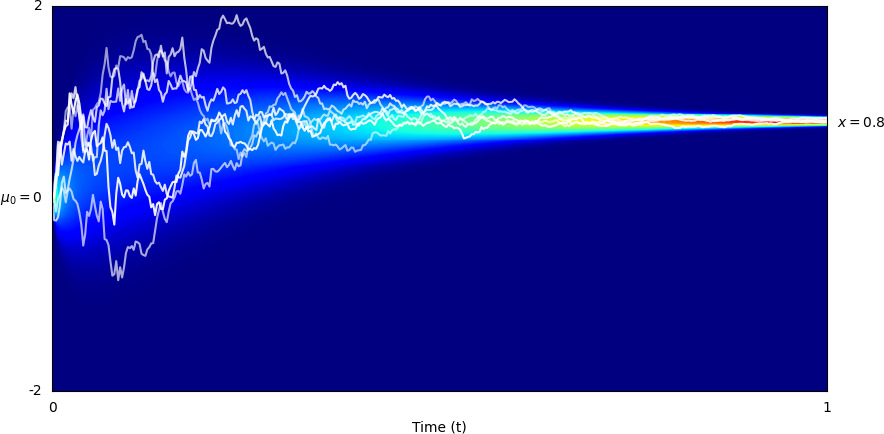}
\caption{\textbf{Bayesian flow for continuous data}. For $x=0.8$, $\sigma_1=0.02$ and $\gamma(t)$ defined as in Eqn.~\ref{cts_gamma_t}, the plot shows stochastic parameter trajectories for the input distribution mean $\mu$ (white lines) superimposed on a log-scale heatmap of the Bayesian flow distribution $p(\theta \mid x; t)$. Note how the trajectories all begin at $\mu_0=0$ then fan out before converging on $x$.}
\label{fig:cts_param_flow}
\end{figure}

\begin{figure}[t]
\centering
\includegraphics[width=\textwidth]{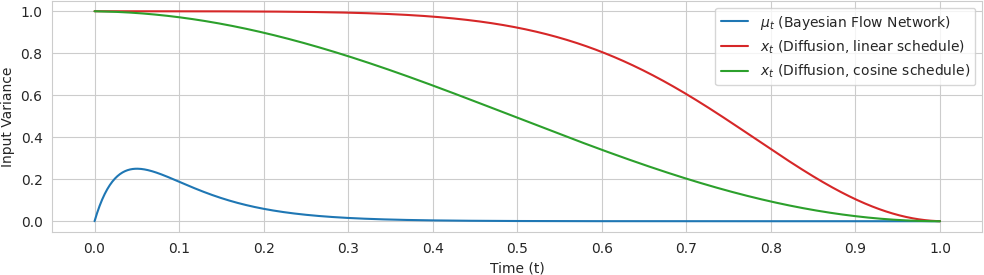}
\caption{\textbf{Input variance for Bayesian Flow Networks and diffusion models}. For $\sigma_1=0.001$ and $\gamma(t)$ defined as in Eqn.~\ref{cts_gamma_t}, the blue line shows the variance $\gamma(t)(1-\gamma(t))$ of the distribution over the input mean $\mu$ as a function of $t$ (see Eq.~\ref{cts_param_flow_dist}). Note that the variance is 0 at $t=0$ (since the input prior $\mu_0$ is deterministic) and becomes small again as $t$ approaches 1 and $\mu$ becomes increasingly concentrated around the data. The green and red lines show the equivalent network input variance for two different noise schedules from the literature (linear~\citep{ ho2020denoising} and cosine~\citep{ nichol2021improved}) during the reverse process of a diffusion model (note that $t$ is reversed relative to diffusion convention). The input variance is much lower for Bayesian Flow Networks.}
\end{figure}
%%%%%%%%%%%%%%%%%%%%%%%%%%%%%%%%%%%%%%%%%%%%%%%%%%%%
\subsection{Output Distribution \texorpdfstring{$\out(\cdot \mid \parsn; t)$}{}}\label{sec:cts_output}
Following standard practice for diffusion models~\citep{song2020score}, the output distribution is defined by reparameterising a prediction of the Gaussian noise vector $\vec{\epsilon} \sim \N{\0{D}}{\I{D}}$ used to generate the mean $\m$ passed as input to the network.
Recall from Eq.~\ref{cts_param_flow_dist} that
\begin{align}
\m \sim  \N{\gamma(t)\x}{\gamma(t)(1-\gamma(t))\I{D}},
\end{align}
and hence
\begin{align}
\m &= \gamma(t)\x + \sqrt{\gamma(t)(1-\gamma(t))} \vec{\epsilon}\\
\implies \x &= \frac{\m}{\gamma(t)}- \sqrt{\frac{1-\gamma(t)}{\gamma(t)}}\vec{\epsilon}.
\end{align}
The network outputs an estimate $\eps(\parsn, t)$ of $\vec{\epsilon}$ and this is transformed into an estimate $\mathbf{\pred{x}}(\parsn, t)$ of $\x$ by 
\begin{align}
\mathbf{\pred{x}}(\parsn, t) = \frac{\m}{\gamma(t)} - \sqrt{\frac{1-\gamma(t)}{\gamma(t)}}\eps(\parsn, t).
\end{align}
Given $\vec{\pred{x}}(\parsn, t)$ the output distribution is
\begin{align}
\out(\x \mid \parsn; t) = \delta(\x-\mathbf{\pred{x}}(\parsn, t))\label{cts_p_dist},
\end{align}
Note that $\gamma(0) = 0$, making the transformation from $\eps(\parsn, t)$ to $\out(\x \mid \parsn; t)$ undefined at $t=0$. 
We therefore set $\out(\x \mid \parsn; t) = \0{D}$ for $t$ under some small threshold $t_{min}$.
Also,  $\mathbf{\pred{x}}(\parsn, t)$ is clipped to lie within the allowed range $[x_{min}, x_{max}]$ for $\x$.
In our experiments $t_{min} = 1\mathrm{e}{-6}$ and $[x_{min}, x_{max}] = [-1, 1]$.
%%%%%%%%%%%%%%%%%%%%%%%%%%%%%%%%%%%%%%%%%%%%%%%%%%%%
\subsection{Sender Distribution \texorpdfstring{$\sender{\cdot}{\x; \alpha}$}{}}\label{sec:cts_sender}
The sender space $\Y = \X = \R$ for continuous data, and the sender distribution is normal with precision $\alpha$:
\begin{align}
\sender{\y}{\x; \alpha} &= \N{\y \mid \x}{\alpha^{-1}\I{D}}\label{cts_q_dist}.
\end{align}
%%%%%%%%%%%%%%%%%%%%%%%%%%%%%%%%%%%%%%%%%%%%%%%%%%%%
\subsection{Receiver Distribution \texorpdfstring{$\rec(\cdot \mid \parsn; t, \alpha)$}{}}
Substituting Eqs.~\ref{cts_p_dist} and  \ref{cts_q_dist} into Eq.~\ref{r_dist},
\begin{align}
\rec(\y \mid \parsn; t, \alpha) &= \E_{\delta(\x'-\mathbf{\pred{x}}(\parsn, t))}\N{\y \mid \x'}{\alpha^{-1}\I{D}}\\
&= \N{\y \mid \mathbf{\pred{x}}(\parsn, t)}{\alpha^{-1}\I{D}}.\label{ctsrecdist}
\end{align}
\begin{figure}[t!]
\includegraphics[width=\textwidth]{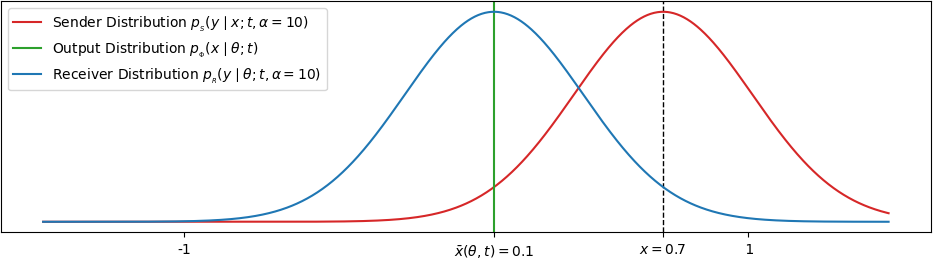}
\caption{\textbf{Sender, output and receiver distributions for continuous data}. Note that the sender and receiver distributions have identical variance and the output distribution is a Dirac delta distribution centred on the network prediction $\pred{x}(\theta, t)$.}
\end{figure}
%%%%%%%%%%%%%%%%%%%%%%%%%%%%%%%%%%%%%%%%%%%%%%%%%%%%
\subsection{Reconstruction Loss \texorpdfstring{$L^r(\x)$}{}}\label{sec:cts_reconstruction}
Truly continuous data requires infinite precision to reconstruct, which makes the reconstruction loss problematic.
However it would be reasonable to assume that either the data is finely discretised (as all information is on a digital computer), or that it contains some noise.
The reconstruction loss for discretised data is presented in Section~\ref{sec:discd_reconstruction}. 
Alternatively, if we assume the presence of normally distributed measurement noise on $\x$, with fixed isotropic variance $\sigma^2$, then a noisy version of the reconstruction loss can be defined as the expected KL divergence between $\N{\x}{\sigma^2\I{D}}$ and the output distribution at $t=1$:
\begin{align}
L^r(\x) &= \E_{\flow(\parsn \mid \x, 1)}\kl{\N{\x}{\sigma^2\I{D}}}{\N{\mathbf{\pred{x}}(\parsn, 1)}{\sigma^2\I{D}}}\\
&= \E_{\flow(\parsn \mid \x, 1)}\frac{1}{2\sigma^2}\left\|\x -\mathbf{\pred{x}}(\parsn, 1)\right\|^2 .
\end{align}
The noise does not directly affect training, as the reconstruction loss is not optimised. 
However the value of $\sigma$ places a natural upper limit on the value that should be chosen for $\sigma_1$: there is no point transmitting the data to greater precision than it was originally measured.
Empirically, we find that when $\sigma_1 < \sigma/2$ the reconstruction loss is very small.
%%%%%%%%%%%%%%%%%%%%%%%%%%%%%%%%%%%%%%%%%%%%%%%%%%%%
\subsection{Discrete-Time Loss \texorpdfstring{$L^{n}(\x)$}{}}\label{sec:cts_disc_t_loss}
From Eqs.~\ref{cts_q_dist} and \ref{ctsrecdist},
\begin{align}
\kl{\sender{\cdot}{\x, \alphat{i}}}{\rec(\cdot \mid \parsnt{i-1}; t_{i-1}, \alphat{i})} &= \kl{\N{\x}{\alphat{i}^{-1}\I{D}}}{\N{\mathbf{\pred{x}}(\parsnt{i-1}, t_{i-1})}{\alphat{i}^{-1}\I{D}}}\\
&= \frac{\alphat{i}}{2}\left\|\x -\mathbf{\pred{x}}(\parsnt{i-1}, t_{i-1})\right\|^2,\label{ctskl}
\end{align}
and from Eqs.~\ref{alpha_i} and \ref{cts_beta_t},
\begin{align}
\alpha_i &= \beta(t_i) - \beta(t_{i-1})\\
&= \sigma_1^{-2i/n} - \sigma_1^{-2(i-1)/n}\\
&= \sigma_1^{-2 i / n} \left(1-\sigma_1^{2/n}\right).
\end{align}
Therefore, substituting into Eq.~\ref{disc_t_loss_exp},
\begin{align}
L^{n}(\x) = \frac{n}{2}\left(1-\sigma_1^{2/n}\right)\E_{i \sim \ui{n},\flow(\parsnt{i-1} \mid \x ; t_{i-1})}  \frac{\left\|\x -\mathbf{\pred{x}}(\parsnt{i-1}, t_{i-1})\right\|^2}{\sigma_1^{2i/n}},\label{n_step_loss_cts}
\end{align}
where $t_{i-1} = (i-1)/{n}$.
%%%%%%%%%%%%%%%%%%%%%%%%%%%%%%%%%%%%%%%%%%%%%%%%%%%%
\subsection{Continuous-time Loss \texorpdfstring{$L^{\infty}(\x)$}{}}\label{sec:ctsctstloss}
Eq.~\ref{convkl} claimed that
\begin{align}
\kl{\sender{\cdot}{\x, \alpha}}{\rec(\cdot \mid \parsn, \alpha, t)} = \kl{\N{g(\x)}{C\alpha^{-1}\I{D}}}{P(\parsn, t) \ast \N{\0{D}}{C\alpha^{-1}\I{D}}},
\end{align}
for some embedding function $g: \X \rightarrow \Y$, constant $C$ and distribution $p_{\parsn}$ over $\Y^D$ with finite mean and variance.
If $g$ is the identity function, $C=1$ and 
\begin{align}
P(\y \mid \parsn, t) &= \delta(\y - \mathbf{\pred{x}}(\parsn, t)),\label{pgycts}
\end{align}
then $P(\parsn, t)$ has finite mean and variance and
\begin{align}
\N{\y \mid g(\x)}{C\alpha^{-1}\I{D}} = \N{\y \mid \x}{\alpha^{-1}\I{D}} &= \sender{\y}{\x; \alpha},\\
P(\y \mid \parsn, t) \ast \N{\0{D}}{C\alpha^{-1}\I{D}} = \N{ \y \mid \mathbf{\pred{x}}(\parsn, t)}{\alpha^{-1}\I{D}} &= \rec(\y \mid \parsn, \alpha, t),
\end{align}
so the claim is true and the continuous-time loss from Eq~\ref{cts_t_loss} applies, with $E[P(\parsn, t)] = \mathbf{\pred{x}}(\parsn, t)$
and $\alpha(t)$ as defined in Eq~\ref{ctsalphat}, yielding
\begin{align}
L^{\infty}(\x) &= -\ln \sigma_1\E_{t \sim U(0,1), \flow(\parsn \mid \x; t)}  \frac{\left\|\x - \mathbf{\pred{x}}(\parsn, t)\right\|^2}{\sigma_1^{2t}}.
\end{align}
%%%%%%%%%%%%%%%%%%%%%%%%%%%%%%%%%%%%%%%%%%%%%%%%%%%%
\subsection{Pseudocode}
Pseudocode for evaluating the $n$-step loss $L^n(\x)$ and continuous-time loss $L^{\infty}(\x)$ for continuous data is presented in Algorithms~\ref{alg:n_step_loss_cts} and \ref{alg:cts_t_loss_cts}, while the sample generation procedure is presented in Algorithm~\ref{alg:samp_gen_cts}.
\begin{algorithm}[H]
\begin{algorithmic}
\LineComment{Note that $\parsn = \{\m, \rho\}$, but $\rho$ is fully determined by $t$}
\LineComment{For our experiments $t_{min} = 1\mathrm{e}{-6}$, $[x_{min}, x_{max}] = [-1, 1]$}
\Function{\lstinline{cts_output_prediction}}{$\m \in \R^D, t \in [0,1], \gamma >\in \R^+$, $t_{min} \in \R^+$, $x_{min}, x_{max} \in \R$}
\If{$t < t_{min}$}
\State $\mathbf{\pred{x}}(\parsn, t) \gets \0{D}$
\Else
\State Input $(\m, t)$ to network, receive $\eps(\parsn, t)$ as output
\State $\mathbf{\pred{x}}(\parsn, t) \gets \frac{\m}{\gamma} - \sqrt{\frac{1-\gamma}{\gamma}}\eps(\parsn, t)$
\State clip $\mathbf{\pred{x}}(\parsn, t)$ to $[x_{min}, x_{max}]$
\EndIf
\State \textbf{Return} $\mathbf{\pred{x}}(\parsn, t)$
\EndFunction
\end{algorithmic}
\end{algorithm}
\begin{algorithm}[H]
\caption{Discrete-Time Loss $L^{n}(\x)$ for Continuous Data}\label{alg:n_step_loss_cts}
\begin{algorithmic}
\State \textbf{Require:} $\sigma_1 \in \R^+$, number of steps $n \in \mathbb{N}$
\State \textbf{Input:} continuous data $\x \in \R^D$
\State $i \sim U\{1, n\}$
\State $t \leftarrow \frac{i-1}{n}$
\State $\gamma \leftarrow 1 - \sigma_1^{2 t}$
\State $\m \sim \N{\gamma \x}{\gamma(1-\gamma)\I{D}}$
\State $\mathbf{\pred{x}}(\parsn, t) \leftarrow \text{\sc{\lstinline{cts_output_prediction}}}(\m, t, \gamma)$
\State $ L^n(\x) \gets \frac{n\left(1-\sigma_1^{2/n}\right)}{2 \sigma_1^{2 i / n}} \left\|\x - \mathbf{\pred{x}}(\parsn, t)\right\|^2$
\end{algorithmic}
\end{algorithm}
\begin{algorithm}[H]
\caption{Continuous-Time Loss $L^{\infty}(\x)$ for Continuous Data}\label{alg:cts_t_loss_cts}
\begin{algorithmic}
\State \textbf{Require:} $\sigma_1 \in \R^+$
\State \textbf{Input:} continuous data $\x \in \R^D$
\State $t \sim U(0,1)$
\State $\gamma \leftarrow 1 - \sigma_1^{2t}$
\State $\m \sim \N{\gamma \x}{\gamma(1-\gamma)\I{D}}$
\State $\mathbf{\pred{x}}(\parsn, t) \gets \text{\sc{\lstinline{cts_output_prediction}}}(\m, t, \gamma)$
\State $ L^{\infty}(\x) \gets -\ln \sigma_1 \sigma_1^{-2t} \left\|\mathbf{\x - \pred{x}}(\parsn, t)\right\|^2$
\end{algorithmic}
\end{algorithm}
\begin{algorithm}[H]
\caption{Sample Generation for Continuous Data}\label{alg:samp_gen_cts}
\begin{algorithmic}
\State \textbf{Require:} $\sigma_1 \in \R^+$, number of steps $n \in \mathbb{N}$
\State $\boldsymbol{\mu} \gets \0{D}$
\State $\rho \gets 1$
\For{$i = 1$ to $n$} 
    \State $t \leftarrow \frac{i-1}{n}$
    \State $\mathbf{\pred{x}}(\parsn, t) \leftarrow \text{\sc{\lstinline{cts_output_prediction}}}(\m, t, 1 - \sigma_1^{2 t})$
    \State $\alpha \gets \sigma_1^{-2 i / n} \left(1-\sigma_1^{2/n}\right)$
    \State $\y \sim \N{\mathbf{\pred{x}}(\parsn, t)}{\alpha^{-1}\I{D}}$
    \State $\m \gets \frac{\rho\boldsymbol{\mu} + \alpha\y}{\rho + \alpha}$
    \State $\rho \gets \rho + \alpha$
\EndFor
\State $\mathbf{\pred{x}}(\parsn, 1) \gets \text{\sc{\lstinline{cts_output_prediction}}}(\m, 1, 1 - \sigma_1^{2})$
\State \textbf{Return} $\mathbf{\pred{x}}(\parsn, 1)$
\end{algorithmic}
\end{algorithm}
%%%%%%%%%%%%%%%%%%%%%%%%%%%%%%%%%%%%%%%%%%%%%%%%%%%%
\section{Discretised Data}\label{sec:discretised}
This section considers continuous data that has been discretised into $K$ bins.
For example, 8-bit images are discretised into 256 bins, 16-bit audio is discretised in $2^{16} = 65,536$ bins.
This data is represented by tiling $[-1, 1]$ into $K$ intervals, each of length $2/K$.
Let $k_{l}$, $\bc{k}$ and $k_{r}$ denote respectively the left, centre and right of interval $k$, and let $\ds{K}$ denote the set of integers from 1 to $K$.
Then for $k \in \ds{K}$,
 \begin{align}
\bc{k} &= \frac{2k - 1}{K} - 1,\\
k_{l} &= \bc{k} - \frac{1}{K},\\
k_{r} &= \bc{k} + \frac{1}{K}.
\end{align}
Let $k(\x) = \left(k(\xdd{1}),\dots, k(\xdd{D})\right) \in \dsd{K}{D}$ be the vector of the indices of the bins occupied by $\x = \left(\didx{x}{1},\dots, \didx{x}{D}\right) \in \R^D$, and let $k_l(\x)$, $k_c(\x)$ and $k_r(\x)$ be the corresponding vectors of left edges, centres and right edges of the bins.
If the data has not already been discretised, we set $\x = k_c(\x)$.
For example if the red channel in an 8-bit RGB image has index 110, it will be represented by the number $\frac{2*(110) - 1}{256} - 1 = -0.14453125$.
Note that each $\didx{x}{d}$ therefore lies in the range $[\frac{1}{K}-1,1-\frac{1}{K}]$ and not $[-1, 1]$.

The input distribution $\inp(\x \mid \parsn)$, prior parameters $\parsnt{0}$, sender distribution $\sender{\y}{\x ; \alpha}$, Bayesian update function $h(\parsnt{i-1}, \y, \alpha)$, Bayesian update distribution $\update(\parsnt{i} \mid \parsnt{i-1}, \x ; \alpha)$, Bayesian flow distribution $\flow(\parsn \mid \x; t)$ and accuracy schedule $\beta(t)$ are all identical to the continuous case described in Section~\ref{sec:cts}. 
It may surprise the reader that the output distribution is discretised while the input, sender and receiver distributions are not.
We made this choice partly for mathematical convenience (Bayesian updates are considerably more complex for discretised distributions;~\citep{austin2021d3pm}) and partly because we suspected that it would easier for the network to interpret continuous means than discrete probabilities as input.
In a similar vein to our argument for standard priors in Sec.~\ref{sec:cts_input}, we remind the reader that the input distribution only serves to inform the network and not directly to model the data; all that matters is that the input parameters contain enough information to allow the network to make accurate predictions.

Section~\ref{sec:cts_disc_t_loss} noted that the level of measurement noise assumed for continuous data should inform the choice of standard deviation $\sigma_1$ for the input distribution at $t=1$ (which in turn defines the accuracy schedule $\beta(t)$).
For discretised data a similar role is played by the width of the discretisation bins, as these place a natural limit on how precisely the data needs to be transmitted.
For example, for $8$-bit data with 256 bins and hence a bin width of $1/128$, setting $\sigma_1 = 1\mathrm{e}{-3}$ corresponds to a final input distribution with standard deviation roughly one eighth of the width of the bin, which should be precise enough for the network to identify the correct bin with very high probability.

One caveat with discretisation is that calculating the loss has $O(K)$ computational cost, which may be prohibitive for very finely discretised data. 
In any case, the benefits of discretisation tend to decrease as the number of bins increases, as we will see in our experiments.
\begin{figure}[t!] 
\includegraphics[width=\textwidth]{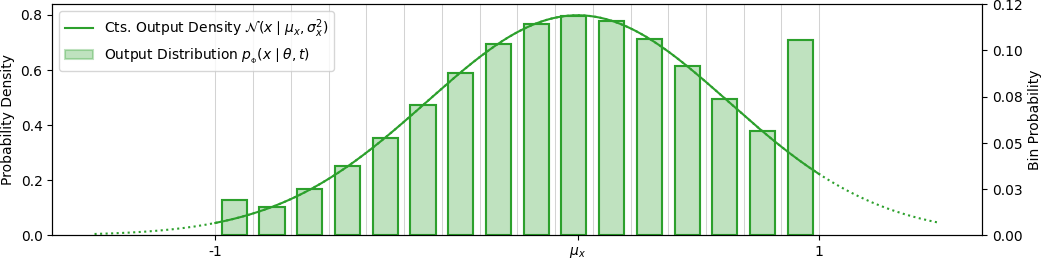}
\caption{\textbf{Output distribution for discretised data}. For univariate data $x$ discretised into $K=16$ bins, the green line shows the continuous distribution $\N{\mu_x}{\sigma^2_x}$ that is discretised to yield the output distribution $\out(x\mid \theta, t)$, as described in Section~\ref{sec:discd_output}. Bin boundaries are marked with vertical grey lines. The heights of the green bars represent the probabilities assigned to the respective bins by $\out(x\mid \theta, t)$. For ease of visualisation these heights are rescaled relative to the probability density, as indicated on the right axis. Note the clipping at $\pm 1$: the area under the dotted green line to the left of $-1$ is added to the probability of the first bin, the area under the dotted green line to the right of $1$ is added to the probability of the last bin.}
\label{fig:discd_p}
\end{figure}
%%%%%%%%%%%%%%%%%%%%%%%%%%%%%%%%%%%%%%%%%%%%%%%%%%%%
\subsection{Output Distribution \texorpdfstring{$\out(\cdot \mid \parsn, t)$}{}}\label{sec:discd_output}
Discretised continuous distributions offer a natural and expressive way to model discretised data with neural networks~\cite{salimans2017pixel}. 
As in Section~\ref{sec:cts_output}, the network outputs $\net(\parsn, t)$ are not used to predict $\x$ directly, but rather to model the Gaussian noise vector $\e$ used to generate the mean sample $\m$ passed as input to the network.

First $\net(\parsn, t)$ is split into two length $D$ vectors, $\m_{\epsilon}$ and $\ln \vec{\sigma}_{\epsilon}$.
Then these are transformed to $\m_{x}$ and $\vec{\sigma}_{x}$ using
\begin{align}
\m_{x} &= \begin{cases}\0{D} & \text{if $t < t_{min}$},\\ \frac{\m}{\gamma(t)} - \sqrt{\frac{1-\gamma(t)}{\gamma(t)}}\m_{\epsilon} & \text{otherwise},\end{cases}\\
\vec{\sigma}_{x} &= \begin{cases}\1{D} & \text{if $t < t_{min}$},\\ \sqrt{\frac{1-\gamma(t)}{\gamma(t)}}\exp(\ln \vec{\sigma}_{\epsilon}) & \text{otherwise}.\end{cases}
\end{align}
For each $d \in \ds{D}$, define the following univariate Gaussian cdf
\begin{align}
F\left(x \mid \mu_x^{(d)}, \sigma_x^{(d)}\right) &= \frac{1}{2}\left[1+\text{erf}\left( \frac{x - \mu_x^{(d)}}{\sigma_x^{(d)}\sqrt{2}}\right)\right],
\end{align}
and clip at $[-1, 1]$ to obtain
\begin{align}
G\left(x \mid \mu_x^{(d)}, \sigma_x^{(d)}\right) = \begin{cases}
0&\text{if $x \leq -1$},\\
1&\text{if  $x \geq 1$},\\
F\left(x \mid \mu_x^{(d)}, \sigma_x^{(d)}\right)&\text{otherwise}.
\end{cases}
\end{align}
Then, for $k \in \ds{K}$,
\begin{align}
\out^{(d)}(k \mid \parsn; t) \defeq G(k_r\mid  \mu^{(d)}_x, \sigma^{(d)}_x)-G(k_l\mid  \mu^{(d)}_x, \sigma^{(d)}_x),
\end{align}
and hence
\begin{align}
\out(\x \mid \parsnt, t) =
\prod_{d=1}^D \out^{(d)}\left(k(\xdd{d})\mid \parsn; t\right).\label{discd_p_dist}
\end{align}
%%%%%%%%%%%%%%%%%%%%%%%%%%%%%%%%%%%%%%%%%%%%%%%%%%%%
\subsection{Receiver Distribution \texorpdfstring{$\rec(\cdot \mid \parsn; t, \alpha)$}{}}
Substituting Eq.~\ref{discd_p_dist} and Eq. \ref{cts_q_dist} into Eq.~\ref{r_dist} gives
\begin{align}
\rec(\y \mid \parsn; t, \alpha) &= \E_{\out(\x' \mid \parsnt, t)}\N{\ydd{d} \mid k_c(\x')}{\alpha^{-1} \I{D}}\\
&= \prod_{d=1}^D \int_{x'}d x' {\out^{(d)}\left(k(x') \mid \parsn; t\right)}\N{\ydd{d} \mid k_c(x')}{\alpha^{-1}}\\
&= \prod_{d=1}^D \sum_{k=1}^K \out^{(d)}(k \mid \parsn; t) \N{\ydd{d} \mid k_c}{\alpha^{-1}}\label{discd_r_dist_1}.
\end{align}
\begin{figure}[t!]
\centering
\begin{subfigure}[b]{\textwidth}
\includegraphics[width=\textwidth]{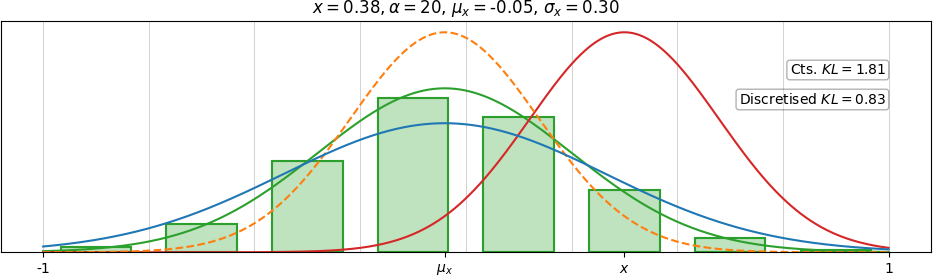}
\end{subfigure}
\begin{subfigure}[b]{\textwidth}
\includegraphics[width=\textwidth]{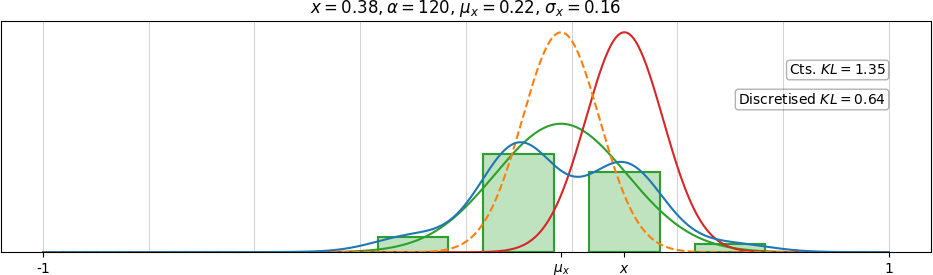}
\end{subfigure}
\begin{subfigure}[b]{\textwidth}
\includegraphics[width=\textwidth]{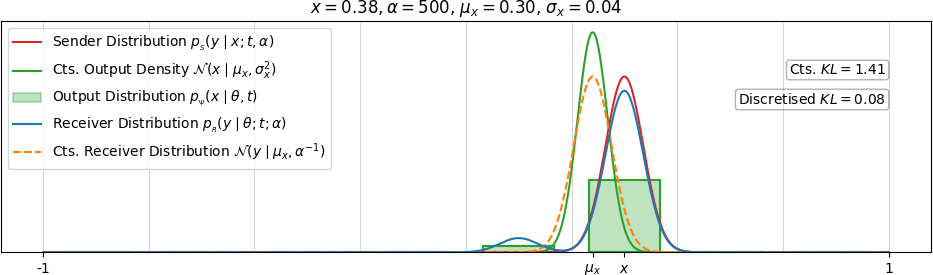}
\end{subfigure}
\caption{\textbf{Sender, output and receiver distributions for discretised data}. For data $x$ discretised into 8 bins, the three plots depict the sender distribution (red line), the discretised output distribution (green bars; heights reflect the probabilities assigned to each bin, rescaled as in Figure~\ref{fig:discd_p}) and receiver distribution (blue line) for progressively increasing values of $\alpha$, and for progressively more accurate predictions of $x$ (both of which typically happen as $t$ increases). Also shown are the continuous distribution $\mathcal{N}(x \mid \mu_x, \sigma^2_x)$ (dotted green line) which is discretized to create the output distribution and the continuous receiver distribution from Section~\ref{sec:cts} (dashed orange line). Bin boundaries are marked with vertical grey lines. Note the KL divergences printed in the top right: taking discretisation into account leads to a lower KL due to the density ``bumps'' at the bin centres where $x$ could be. The advantage of discretisation becomes more pronounced as the prediction gets closer to $x$ and more of the probability mass is concentrated in the correct bin.}
\end{figure}
%%%%%%%%%%%%%%%%%%%%%%%%%%%%%%%%%%%%%%%%%%%%%%%%%%%%
\subsection{Reconstruction Loss \texorpdfstring{$L^r(\x)$}{}}\label{sec:discd_reconstruction}
The reconstruction loss for discretised data is
\begin{align}
L^r(\x) &= -\E_{\flow(\parsn \mid \x, 1)} \ln \out(\x \mid \parsn; 1)\\
&= -\E_{\flow(\parsn \mid \x, 1)}\sum_{d=1}^D \ln \out^{(d)}\left(k(\xdd{d}) \mid \parsn; 1 \right).
\end{align}
%%%%%%%%%%%%%%%%%%%%%%%%%%%%%%%%%%%%%%%%%%%%%%%%%%%%
\subsection{Discrete-time Loss \texorpdfstring{$L^{n}(\x)$}{}}
From Eqs.~\ref{cts_q_dist} and \ref{discd_r_dist_1},
\begin{align}
&\kl{\sender{\cdot}{\x, \alphat{i}}}{\rec(\cdot \mid \parsnt{i-1}; t_{i-1}, \alphat{i})}\\ 
&\qquad\qquad\qquad=\kl{\N{\x}{\alphat{i}^{-1}\I{D}}}{\prod_{d=1}^D\sum_{k=1}^K \out^{(d)}(k \mid \parsnt{i-1}, t_{i-1}) \N{k_c}{\alphat{i}^{-1}}},
\end{align}
which cannot be calculated in closed form, but can be estimated with Monte-Carlo sampling.
Substituting into Eq.~\ref{disc_t_loss_exp},
\begin{align}
&L^{n}(\x) = n \E_{i \sim \ui{n},\flow(\parsn \mid \x ; t_{i-1}),\N{\y \mid \x}{\alphat{i}^{-1}\I{D}}} \ln \N{\y \mid \x}{\alphat{i}^{-1}\I{D}}\\
&\qquad\qquad\qquad\qquad- \sum_{d=1}^D\ln \left(\sum_{k=1}^K \out^{(d)}(k \mid \parsn, t_{i-1}) \N{\ydd{d} \mid k_c}{\alphat{i}^{-1}}\right).\label{discd_disc_t_loss_exp}
\end{align}
%%%%%%%%%%%%%%%%%%%%%%%%%%%%%%%%%%%%%%%%%%%%%%%%%%%%
\subsection{Continuous-time Loss \texorpdfstring{$L^{\infty}(\x)$}{}}
Justifying the claim made in Eq.~\ref{convkl} follows almost the same reasoning here as in Section~\ref{sec:ctsctstloss}, with $C=1$ and $g$ the identity function. 
The only difference is that
\begin{align}
P(\y \mid \parsn; t) = \prod_{d=1}^D \sum_{k=1}^K \out^{(d)}(k \mid \parsn, t) \delta(\ydd{d} - \bc{k}),
\end{align}
which clearly has finite variance and mean.
Since
\begin{align}
P(\y \mid \parsn, t) \ast \N{\0{D}}{C\alpha^{-1}\I{D}} = \rec(\y \mid \parsn, \alpha, t),
\end{align}
the claim holds and the continuous time loss from Eq~\ref{cts_t_loss} can be applied with
\begin{align}
E[P(\parsn, t)] = \left(\sum_{k=1}^K p^{(1)}(k \mid \parsn, t) k_c, \dots,\sum_{k=1}^K p^{(D)}(k \mid \parsn, t) k_c\right) \defeq \mathbf{\pred{k}}(\parsn, t),
\end{align}
and $\alpha(t)$ as defined in Eq~\ref{ctsalphat}, yielding
\begin{align}
L^{\infty}(\x) &= -\ln \sigma_1\E_{t \sim U(0,1), \flow(\parsn \mid \x; t)}  \frac{\left\|\x -\mathbf{\pred{k}}(\parsn, t) \right\|^2}{\sigma_1^{2t}}.
\end{align}
Note that $\mathbf{\pred{k}}(\parsn, t)$ is a function of the complete discretised distribution $\out(\x \mid \parsn, t)$, hence $L^{\infty}(\x)$ depends on both $\m_{\x}$ and $\vec{\sigma}_{\x}$, and not only on $\m_{\x}$, as for continuous data. This also means that calculating $L^{\infty}(\x)$ has $O(K)$ computational cost for discretised data.
%%%%%%%%%%%%%%%%%%%%%%%%%%%%%%%%%%%%%%%%%%%%%%%%%%%%
\subsection{Pseudocode}
Pseudocode for evaluating the discrete-time loss $L^n(\x)$ and continuous-time loss $L^{\infty}(\x)$ for discretised data is presented in Algorithms~\ref{alg:n_step_loss_discd} and \ref{alg:cts_t_loss_discd}, while sample generation is presented in Algorithm~\ref{alg:samp_gen_discd}.
\begin{algorithm}[H]
\begin{algorithmic}
\Function{\lstinline{discretised_cdf}}{$\mu \in \R, \sigma \in \R^+, x \in \R$}
\State $F(x) \gets \frac{1}{2}\left[1+\text{erf}\left( \frac{x - \mu}{\sigma \sqrt{2}}\right)\right]$
\State $G(x) \gets \begin{cases}
0&\text{ if } x \leq -1\\
1&\text{ if } x \geq 1\\
F(x) &\text{ otherwise} \end{cases}$
\State \textbf{Return} $G(x)$
\EndFunction
\end{algorithmic}
\end{algorithm}
\begin{algorithm}[H]
\begin{algorithmic}
\LineComment{For our experiments $t_{min} = 1\mathrm{e}{-6}$}
\LineComment{$k_{l} = \frac{2(k-1)}{K} - 1$, $k_{r} = \frac{2k}{K} - 1$}
\Function{\lstinline{discretised_output_distribution}}{$\m \in \R^D, t \in [0,1], K \in \mathbb{N}, \gamma \in \R^+$, $t_{min} \in \R^+$}.
\If{$t < t_{min}$}
\State $\m_{x} \gets \0{D}$
\State $\vec{\sigma}_{x} \gets \1{D}$
\Else
\State Input $(\m, t)$ to network, receive $(\m_{\epsilon}, \ln \vec{\sigma}_{\epsilon})$ as output
\State $\m_{x} \gets \frac{\m}{\gamma} - \sqrt{\frac{1-\gamma}{\gamma}}\m_{\epsilon}$
\State $\vec{\sigma}_{x} \gets \sqrt{\frac{1-\gamma}{\gamma}}\exp(\ln \vec{\sigma}_{\epsilon})$
\EndIf
\For{$d \in \ds{D}$, $k \in \ds{K}$}
\State $\out^{(d)}(k \mid \parsn; t) \gets \text{\sc{\lstinline{discretised_cdf}}}(\mu_x^{(d)}, \sigma_x^{(d)}, k_r) - \text{\sc{\lstinline{discretised_cdf}}}(\mu_x^{(d)}, \sigma_x^{(d)}, k_l)$
\EndFor
\State \textbf{Return} $\outn(\cdot \mid \parsn; t)$
\EndFunction
\end{algorithmic}
\end{algorithm}
\begin{algorithm}[H]
\caption{Discrete-Time Loss $L^{n}(\x)$ for Discretised Data}\label{alg:n_step_loss_discd}
\begin{algorithmic}
\LineComment{$k_{c} = \frac{2k-1}{K} - 1$}
\State \textbf{Require:} $\sigma_1 \in \R^+$, number of steps $n \in \mathbb{N}$, number of bins $K \in \mathbb{N}$
\State \textbf{Input:} discretised data $\x \in [\frac{1}{K}-1,1-\frac{1}{K}]^D$
\State $i \sim U\{1, n\}$
\State $t \leftarrow \frac{i-1}{n}$
\State $\gamma \leftarrow 1 - \sigma_1^{2 t}$
\State $\m \sim \N{\gamma \x}{\gamma(1-\gamma)\I{D}}$
\State $\alpha \gets \sigma_1^{-2 i / n} \left(1-\sigma_1^{2/n}\right)$
\State $\y \sim \N{\x}{\alpha^{-1}\I{D}}$
\State $\outn(\cdot \mid \parsn; t) \leftarrow \text{\sc{\lstinline{discretised_output_distribution}}}(\m, t, K, \gamma)$
\State $ L^n(\x) \gets n \left[\ln \N{\y \mid \x}{\alpha^{-1}\I{D}} - \sum_{d}\ln \left(\sum_{k} \out^{(d)}(k \mid \parsn; t) \N{\ydd{d} \mid k_c}{\alpha^{-1}}\right)\right]$
\end{algorithmic}
\end{algorithm}
\begin{algorithm}[H]
\caption{Continuous-Time Loss $L^{\infty}(\x)$ for Discretised Data}\label{alg:cts_t_loss_discd}
\begin{algorithmic}
\State \textbf{Require:} $\sigma_1 \in \R^+$, number of bins $K \in \mathbb{N}$
\State \textbf{Input:} discretised data $\x \in [\frac{1}{K}-1,1-\frac{1}{K}]^D$
\State $t \sim U(0,1)$
\State $\gamma \leftarrow 1 - \sigma_1^{2t}$
\State $\m \sim \N{\gamma \x}{\gamma(1-\gamma)\I{D}}$
\State $\outn(\cdot \mid \parsn; t) \leftarrow \text{\sc{\lstinline{discretised_output_distribution}}}(\m, t, K, \gamma)$
\State $\mathbf{\pred{k}}(\parsn, t) \gets \left(\sum_k \out^{(1)}(k \mid \parsn; t)k_c,\dots, \sum_k \out^{(D)}(k \mid \parsn; t) k_c\right)$
\State $ L^{\infty}(\x) \gets -\ln \sigma_1 \sigma_1^{-2t}\left\|\x -\mathbf{\pred{k}}(\parsn, t) \right\|^2$
\end{algorithmic}
\end{algorithm}
\begin{algorithm}[H]
\caption{Sample Generation for Discretised Data}\label{alg:samp_gen_discd}
\begin{algorithmic}
\LineComment{$\vec{k}_{c} = \left(k^{(1)}_c,\dots,k^{(D)}_c\right)$}
\State \textbf{Require:} $\sigma_1 \in \R^+$, number of steps $n \in \mathbb{N}$, number of bins $K \in \mathbb{N}$
\State $\boldsymbol{\mu} \gets \0{D}$
\State $\rho \gets 1$
\For{$i = 1$ to $n$} 
    \State $t \leftarrow \frac{i-1}{n}$
    \State $\k \sim \text{\sc{\lstinline{discretised_output_distribution}}}(\m, t, K, 1 - \sigma_1^{2 t})$
    \State $\alpha \gets \sigma_1^{-2 i / n} \left(1-\sigma_1^{2/n}\right)$
    \State $\y \sim \N{\k_c}{\alpha^{-1}\I{D}}$
    \State $\m \gets \frac{\rho\boldsymbol{\mu} + \alpha\y}{\rho + \alpha}$
    \State $\rho \gets \rho + \alpha$
\EndFor
\State $\k \sim \text{\sc{\lstinline{discretised_output_distribution}}}(\m, 1, K, 1 - \sigma_1^{2})$
\State \textbf{Return} $\k_c$
\end{algorithmic}
\end{algorithm}
%%%%%%%%%%%%%%%%%%%%%%%%%%%%%%%%%%%%%%%%%%%%%%%%%%%%
\section{Discrete Data}\label{sec:discrete}
We now consider discrete data in which no meaningful order or distance exists between the classes, unlike the discretised continuous data covered in the previous section.
Some obvious examples are text characters, classification labels or any binary data. 
In this context the data is represented as a $D$ dimensional vector of class indices:  $\x = \left(\didx{x}{1},\dots, \didx{x}{D}\right) \in \dsd{K}{D}$, where $\ds{K}$ is the set of integers from $1$ to $K$.
%%%%%%%%%%%%%%%%%%%%%%%%%%%%%%%%%%%%%%%%%%%%%%%%%%%%
\subsection{Input Distribution \texorpdfstring{$\inp(\cdot \mid \parsn)$}{}}\label{sec:disc_input}
For discrete data, the input distribution is a factorised categorical over the class indices.
Let $\parsn = \left(\parsdd{1},\dots,\parsdd{D}\right) \in [0,1]^{KD}$ with $\parsdd{d}= \left(\pars_1^{(d)},\dots,\pars_K^{(d)}\right) \in \Delta^{K-1}$, where $\pars_k^{(d)}$ is the probability assigned to class $k$ for variable $d$. 
Then
\begin{align}
\inp(\x \mid \parsn) = \prod_{d=1}^D  \pars_{\didx{x}{d}}^{(d)}.
\end{align}
The input prior is uniform with
\begin{align}
\parsnt{0} = \vec{\frac{1}{K}}\label{disc_input_prior},
\end{align}
where $\vec{\frac{1}{K}}$ is the length $KD$ vector whose entries are all $\frac{1}{K}$.
We chose a uniform prior---rather than an empirical prior fit to the training data---for the same reasons we chose a standard normal prior for continuous data: it's mathematically simpler, and the disparity between the true prior and the simple prior can easily be corrected by the network.
%%%%%%%%%%%%%%%%%%%%%%%%%%%%%%%%%%%%%%%%%%%%%%%%%%%%
\subsection{Output Distribution \texorpdfstring{$\out(\cdot \mid \parsn; t)$}{}}\label{sec:disc_output}
Given data $\x$, network inputs $\parsn, t$ and corresponding network outputs $\net(\parsn, t) = \left(\didx{\net}{1}(\parsn, t),\dots,\didx{\net}{D}(\parsn, t)\right)\\ \in \R^{KD}$, the output distribution for discrete data is as follows:
\begin{align}
\out^{(d)}(k \mid \parsn; t) &= \left(\text{softmax}(\didx{\net}{d}(\parsn, t))\right)_k,\\
\out(\x \mid \parsn; t) &= \prod_{d=1}^D \out^{(d)}(\xdd{d} \mid \parsn; t).\label{disc_pred_dist}
\end{align}
Note that for binary data only the probability $\theta^{(d)}_1$ that $k=1$ is fed into the network, on the grounds that the probability of $k=2$ can easily be inferred from $\theta^{(d)}_2 = 1 - \theta^{(d)}_1$.
The output distribution for binary data is determined by applying the logistic sigmoid function elementwise to the length $D$ output vector to get the probability for $k=1$:
\begin{align}
\out^{(d)}(1 \mid \parsn; t) = \sigma\left(\didx{\net}{d}(\parsn, t))\right),
\end{align}
where
\begin{align}
\sigma(x) =  \frac{1}{1-e^{-x}},
\end{align}
then inferring the probabilities for $k=2$ from
\begin{align}
\out^{(d)}(2 \mid \parsn; t) = 1 - \out^{(d)}(1 \mid \parsn; t).
\end{align}
In principle one class could also be removed from the inputs and outputs when $K > 2$ and inferred from the others.
However this would require the network to internalise a slightly more sophisticated inference procedure that could potentially slow down learning.
We therefore followed deep-learning convention and included a redundant input and output unit for $K>2$.

All probabilities are rescaled to the range $[-1, 1]$ by multiplying by two then subtracting one before feeding them into the network.
%%%%%%%%%%%%%%%%%%%%%%%%%%%%%%%%%%%%%%%%%%%%%%%%%%%%
\subsection{Sender Distribution \texorpdfstring{$\sender{\cdot}{\x; \alpha}$}{}}\label{sec:disc_sender}
Given $\omega \in [0,1]$, and a vector of $D$ class indices $\k = \left(\didx{k}{1},\dots,\didx{k}{D}\right) \in \dsd{K}{D}$, let
\begin{align}
p(\didx{k}{d} \mid \didx{x}{d}; \omega) &\defeq \frac{1-\omega}{K} + \omega \delta_{\didx{k}{d} \didx{x}{d}}\label{q_def},
\end{align}
where $\delta_{i j}$ is the Kronecker delta function.
Clearly $p(\didx{k}{d} \mid \didx{x}{d}; \omega) \geq 0\ \forall k$ and $\sum_{k=1}^K p(\didx{k}{d} \mid \didx{x}{d}; \omega) = 1$, so the vector
\begin{align}
a(\didx{x}{d}, \omega) \defeq \left(p(1 \mid \didx{x}{d}; \omega),\dots,p(K \mid \didx{x}{d}; \omega)\right),
\end{align}
defines a valid distribution over $K$ classes.
To simplify notation we will from now on drop the superscripts and refer to $\didx{x}{d}$ as $x$, $p(\didx{k}{d} \mid \didx{x}{d}; \omega)$ as $p(k \mid x; \omega)$ and so on, except where necessary to remove ambiguity.

Consider a vector of integer counts $c = (c_1,\dots,c_K) \in \dsd{m}{K}$, corresponding to the number of times each of the $K$ classes is observed among $m$ independent draws from $a(x, \omega)$.
Then the probability of observing $c$ is given by the following multinomial distribution:
\begin{align}
p(c \mid x, \omega) &= \text{Multi}(m, a(x, \omega))\label{multi_def}\\
&= \frac{m!}{c_1!\dots c_K!} \prod_{k=1}^K \left(p(k \mid x; \omega)\right)^{c_k}\\
&= \frac{m!}{c_1!\dots c_K!} \prod_{k=1}^K \left(\frac{1-\omega}{K} + \omega\delta_{k x}\right)^{c_k}.\label{count_dist}
\end{align}
Now consider the fraction $c_k/m$ of observations of class $k$ in $c$.
Clearly
\begin{align}
\lim_{m\rightarrow \infty} \frac{c_k}{m} = p(k \mid x; \omega),
\end{align}
meaning that for any finite $\omega$ it would be possible to deduce from $c$ what the value of $x$ is if $m$ is sufficiently large.
However as $\omega$ shrinks, $p(k \mid x; \omega)$ becomes closer to uniform, meaning that a larger $m$ is required to unambigously identify $x$ from $c$. 
By defining the accuracy $\alpha \defeq m\omega^2$ and sending $m \rightarrow \infty$ (and hence $\omega \rightarrow 0$ for any finite $\alpha$), $p(c \mid x, \omega)$ can therefore be used to define a continuous-valued sender distribution that smoothly varies from totally uninformative at $\alpha=0$ to totally informative as $\alpha \rightarrow \infty$, like the sender distribution for continuous data.

It can be proved from the central limit theorem that for any set of discrete probabilities $p = \{p_1,\dots,p_K\}$, where $0 < p_k < 1$ $\forall k$, that if $c \sim \text{Multi}(m, p)$ then in the limit $m \rightarrow \infty$ the following result holds~\cite{georgii2008stochastics}:
\begin{align}
&\frac{c - m p}{\sqrt{m p}} \sim \N{0}{\I{K}},
% \implies &c_k \sim \N{m p_k}{m p_k}
\end{align}
where $\I{K}$ is the $K \times K$ identity matrix.
Therefore
\begin{align}
\lim_{m\rightarrow \infty} p(c_k \mid x,\omega) &= \N{c_k \mid m p(k \mid x; \omega)}{m p(k \mid x; \omega)}\\
&= \frac{1}{\sqrt{2\pi m p(k \mid x; \omega)}}\exp\left(\frac{-\left[c_k - m p(k \mid x,\omega)\right]^2}{2 m p(k \mid x; \omega)}\right).
\end{align}
Now define
\begin{align}
\xi &\defeq 1 + \frac{\omega K}{1-\omega}\label{gamma_def}.
\end{align}
And the length $K$ sender sample $y = (y_1,\dots,y_K)$ as
\begin{align}
y_k &\defeq \left(c_k - \frac{m}{K}\right) \ln \xi\label{y_def}.
\end{align}
Note that $y$, unlike $x$, is continuous ($\Y = \R^{K}, \X = \{1,K\}$), and that $\left(c - \frac{m}{K}\right)$ measures the number of times each class is observed, minus the average number of observations per class.
Intuitively, $y$ provides information about the relative concentration of the classes among the counts, with (since $\ln \xi > 0$) positive values for classes observed more frequently than the mean and negative values for those observed less frequently than the mean. 
As $m \omega^2$ grows the concentration increases around the true class, and hence $y$ become more informative about $x$.

Rearranging Eq.~\ref{y_def},
\begin{align}
c_k &= \frac{y_k}{\ln\xi} + \frac{m}{K}\\
\implies \frac{d c_k}{d y_k} &= \frac{1}{\ln\xi},
\end{align}
which we can use for the following change of variables:
\begin{align}
p(y_k \mid x,\omega) &= \left|\frac{d c_k}{d y_k} \right|p(c_k \mid x, \omega)\\
&= \frac{1}{\ln\xi\sqrt{2\pi m p(k \mid x,\omega)}}\exp\left(\frac{-\left[\frac{y_k}{\ln\xi} + \frac{m}{K} - m p(k \mid x,\omega)\right]^2}{2 m p(k \mid x,\omega)}\right)\label{above},
\end{align}
where we have used the fact that $\xi \geq 1$ and hence $\frac{d c_k}{d y_k} \geq 0$. 
Recall that $\alpha = m\omega^2$ and hence $m = \frac{\alpha}{\omega^2}$,
 which can be substituted into the above to yield
\begin{align}
p(y_k \mid x,\omega) &= \frac{1}{\frac{1}{\omega}\ln\xi}\frac{1}{\sqrt{2\pi \alpha p(k \mid x,\omega)}}\exp\left(\frac{-\left[\frac{y_k}{\frac{1}{\omega}\ln \xi} + \frac{\alpha}{\omega}\left(\frac{1}{K}-p(k \mid x,\omega)\right)\right]^2}{2\alpha p(k \mid x,\omega)}\right).
\end{align}
Substituting from Eq.~\ref{q_def},
\begin{align}
\frac{1}{K}-p(k \mid x,\omega) = \omega\left(\frac{1}{K}-\delta_{kx},\right),
\end{align}
and hence
\begin{align}
p(y_k \mid x,\omega) &= \frac{1}{\frac{1}{\omega}\ln\xi}\frac{1}{\sqrt{2\pi \alpha p(k \mid x,\omega)}}\exp\left(\frac{-\left[\frac{y_k}{\frac{1}{\omega}\ln \xi} - \alpha\left(\delta_{k x} - \frac{1}{K}\right)\right]^2}{2\alpha p(k \mid x,\omega)}\right)\label{p_y_i_omega}.
\end{align}
Applying the identity $\ln(1+x) = \sum_{n=1}^{\infty} \frac{(-1)^{n-1}}{n}x^n$ for $|x| < 1$  to $\ln \xi = \ln\left(1 + \frac{\omega K}{1-\omega} \right)$ it can be seen that
\begin{align}
\ln \xi &\in  \frac{\omega K}{1-\omega} + O(\omega^2),
\end{align}
and hence
\begin{align}
\lim_{\omega \rightarrow 0} \frac{1}{\omega}\ln \xi &= K.\label{gamma_limit}
\end{align}
Furthermore, it follows directly from Eq.~\ref{q_def} that
\begin{align}
\lim_{\omega \rightarrow 0} p(k \mid x,\omega) = \frac{1}{K}\ \forall k \in \ds{K}\label{q_limit}.
\end{align}
Now define
\begin{align}
\sender{y_k}{x;\alpha} \defeq \lim_{\omega \rightarrow 0}p(y_k \mid x,\omega).
\end{align}
Plugging Eq.~\ref{gamma_limit} and \ref{q_limit} into Eq.~\ref{p_y_i_omega},
\begin{align}
\sender{y_k}{x;\alpha} &= \frac{1}{K\sqrt{2\pi \alpha \frac{1}{K}}}\exp\left(\frac{-\left[\frac{y_k}{K} - \alpha\left(\delta_{k x} - \frac{1}{K}\right)\right]^2}{2\alpha \frac{1}{K}}\right)\\
&= \frac{1}{\sqrt{2\pi \alpha K}}\exp\left(\frac{-\left[y_k - \alpha\left(K\delta_{k x} - 1\right)\right]^2}{2\alpha K}\right)\\
&= \N{\alpha\left(K\delta_{k x} - 1\right)}{\alpha K}\label{y_i_dist}.
\end{align}
Restoring the superscript,
\begin{align}
\sender{\ydd{d}}{\xdd{d};\alpha} &= \N{\alpha\left(K \oh{\xdd{d}}{K}- \1{K}\right)}{\alpha K \I{K}}\label{disc_q_def_uni},
\end{align}
where $\1{K}$ is a vector of ones, $\I{K}$ is the identity matrix and $\oh{j}{K}\in \R^{K}$ is the projection from the class index $j$ to the length $K$ one-hot vector defined by $(\oh{j}{K})_k = \delta_{j k}$, and therefore
\begin{align}
\sender{\y}{\x;\alpha} = \N{\y \mid \alpha\left(K \oh{\x}{KD} - \1{KD}\right)}{\alpha K \I{KD}}\label{disc_q_dist},
\end{align}
where $\oh{\x}{KD} \defeq \left(\oh{\xdd{1}}{K},\dots,\oh{\xdd{D}}{K}\right) \in \R^{KD}$.
%%%%%%%%%%%%%%%%%%%%%%%%%%%%%%%%%%%%%%%%%%%%%%%%%%%%
\subsection{Receiver Distribution \texorpdfstring{$\rec(\cdot \mid \parsn; t, \alpha)$}{}}
Substituting Eq.~\ref{disc_pred_dist} and Eq. \ref{disc_q_dist} into Eq.~\ref{r_dist} gives the following receiver distribution for dimension $d$:
\begin{align}
\rec^{(d)}(\ydd{d} \mid \parsn; t, \alpha) &= \sum_{k=1}^K \out^{(d)}(k \mid \parsn; t) \N{\alpha\left(K \oh{k}{K}- \1{K}\right)}{\alpha K \I{K}}\label{disc_r_dist_uni},\\
\rec(\y \mid \parsn; t, \alpha) &= \prod_{d=1}^D \rec^{(d)}(\ydd{d} \mid \parsn; t, \alpha).\label{disc_r_dist}
\end{align}
%%%%%%%%%%%%%%%%%%%%%%%%%%%%%%%%%%%%%%%%%%%%%%%%%%%%
\subsection{Bayesian Update Function \texorpdfstring{$h(\parsnt{i-1}, \y, \alpha)$}{}}
Recall from Section~\ref{sec:disc_input} that $(\theta_{i-1})^{(d)}_k$ is the probability assigned to $x^{(d)}=k$ by $p(x^{(d)} \mid \theta_{i-1})$.
Dropping the superscript and returning to the count distribution $p(c \mid x, \omega)$ defined in Eq.~\ref{multi_def}, the posterior probability that $x=k$ after observing $c$ is
\begin{align}
p(k \mid c; \omega) &= \frac{p (c \mid k; \omega) (\theta_{i-1})_k}{\sum_{k'=1}^K p(c \mid k';\omega)(\theta_{i-1})_{k'}}.\label{disc_bayes}
\end{align}
Substituting Eq.~\ref{count_dist} into Eq.~\ref{disc_bayes} and cancelling terms in the enumerator and denominator,
\begin{align}
p(k\mid c;\omega) &= \frac{\left[\frac{1-\omega}{K}\right]^{m-c_k}\left[\frac{1-\omega}{K} + \omega\right]^{c_k} (\theta_{i-1})_k}{ \sum_{k'=1}^K{\left[\frac{1-\omega}{K}\right]^{m-c_{k'}}\left[\frac{1-\omega}{K} +\omega \right]^{c_{k'}}(\theta_{i-1})_{k'}}}\\
&=  \frac{\left[\frac{1-\omega}{K}\right]^{m}\left[1 + \frac{\omega K}{1-\omega}\right]^{c_k}(\theta_{i-1})_k}{ \left[\frac{1-\omega}{K}\right]^{m}\sum_{k'=1}^K{\left[1 + \frac{\omega K}{1-\omega}\right]^{c_{k'}}(\theta_{i-1})_{k'}}}\\
&=  \frac{\left[1 + \frac{\omega K}{1-\omega}\right]^{c_k}(\theta_{i-1})_k}{ \sum_{k'=1}^K{\left[1 + \frac{\omega K}{1-\omega}\right]^{c_{k'}}(\theta_{i-1})_{k'}}}\\
&= \frac{\xi^{c_k}(\theta_{i-1})_k}{ \sum_{k'=1}^K{\xi^{c_{k'}}(\theta_{i-1})_{k'}}}\label{post_prob}.
\end{align}
Now define
\begin{align}
h(\theta, y) &\defeq \frac{e^y\theta}{\sum_{k=1}^K e^{y_{k}}\theta_{k}}\label{disc_update_param_def}.
\end{align}
Substituting the definition of $y_k$ from Eq.~\ref{y_def} into the definition of $h(\theta, y)$ from Eq.~\ref{disc_update_param_def},
\begin{align}
\left(h(\theta_{i-1}, y)\right)_k &= \frac{\exp(-\frac{m}{K} \ln \xi)\exp(c_k\ln \xi )(\theta_{i-1})_k}{\exp(-\frac{m}{K} \ln \xi)\sum_{k'=1}^K \exp(c_{k'} \ln \xi )(\theta_{i-1})_{k'}}\\
&= \frac{\exp(\ln \xi^{c_k} )(\theta_{i-1})_k}{\sum_{k'=1}^K \exp(\ln \xi^{c_{k'}})(\theta_{i-1})_{k'}}\\
&= \frac{\xi^{c_k}(\theta_{i-1})_k}{\sum_{k'=1}^K \xi^{c_{k'}}(\theta_{i-1})_{k'}},\\
\end{align}
and hence, from Eq.~\ref{post_prob},
\begin{align}
h(\theta_{i-1}, y)_k = p(k\mid c;\omega).
\end{align}
Therefore in the limit $m\rightarrow \infty$ with $m\omega^2 = \alpha$, the stochastic parameter update from $\theta_{i-1}$ to $\theta_{i}$ induced by drawing $c$ from $\text{multi}(m, a(x, \omega))$ can be sampled by first drawing $y$ from $\sender{\cdot}{x,\alpha}$ then setting $\theta_{i} = h(\theta_{i-1}, y)$.
Hence the Bayesian update function is 
\begin{align}
h(\parsnt{i-1}, \y, \alpha) \defeq \frac{e^{\y}\parsnt{i-1}}{\sum_{k=1}^K e^{\y_k}(\parsnt{i-1})_{k}},\label{disc_param_update_function}
\end{align}
where the redundant parameter $\alpha$ has been included for consistency with the update function for continuous data.
%%%%%%%%%%%%%%%%%%%%%%%%%%%%%%%%%%%%%%%%%%%%%%%%%%%%
\subsection{Bayesian Update Distribution \texorpdfstring{$\update(\cdot \mid \parsnt{i-1}, \x; \alpha)$}{}}
Substituting Eqs.~\ref{disc_q_dist} and \ref{disc_param_update_function} into Eq.~\ref{param_update_dist},
\begin{align}
\update(\parsn \mid \parsnt{i-1}, \x; \alpha) &= \E_{\N{\y \mid \alpha\left(K \oh{\x}{KD} - \1{KD}\right)}{\alpha K \I{KD}}} \delta\left(\parsn - \frac{e^{\y}\parsnt{i-1}}{\sum_{k=1}^K e^{\y_k}(\parsnt{i-1})_{k}}\right).\label{disc_par_update_def}
\end{align}
%%%%%%%%%%%%%%%%%%%%%%%%%%%%%%%%%%%%%%%%%%%%%%%%%%%%
\subsection{Additive Accuracies}\label{sec:disc_additive}
It follows from the definition of the update distribution that if $y_a$ is drawn from $\sender{\cdot}{x; \alpha_a}$ then $\parst{i-1} = h(y_a, \parst{i-2})$ is drawn from $p(\cdot \mid \parst{i-2}, x; \alpha_a)$.
Furthermore, if $y_b$ is drawn from $\sender{\cdot}{x; \alpha_b}$ then $\parst{i} = h(y_b, \parst{i-1}) = h(y_b, h(y_a, \parst{i-2}))$ is drawn from $\E_{\update(\parst{i-1} \mid \parst{i-2}, x; \alpha_a)} \update(\parst{i} \mid \parst{i-1}, x; \alpha_b)$.
Substituting the definition of $h$ from Eqn~\ref{disc_update_param_def},
\begin{align}
h(y_b, h(y_a, \theta_{i-2})) &= \frac{\exp(y_b) \frac{\exp(y_a)\theta_{i-2}}{\sum_{k'=1}^K\exp\left((y_a)_{k'}\right)(\theta_{i-2})_{k'}}}{\sum_{k=1}^K \exp\left((y_b)_k\right)\frac{\exp\left((y_a)_k\right)(\theta_{i-2})_k}{\sum_{k'=1}^K\exp\left((y_a)_{k'}\right)(\theta_{i-2})_{k'}}}\\
&= \frac{\exp(y_b) \exp(y_a)\theta_{i-2}}{\sum_{k=1}^K \exp\left((y_b)_k\right) \exp\left((y_a)_k\right)(\theta_{i-2})_k}\\
&= \frac{\exp(y_a + y_b)\theta_{i-2}}{\sum_{k=1}^K \exp\left((y_a+y_b)_k\right)(\theta_{i-2})_k}\\
&= h(y_a+y_b, \theta_{i-2}).
\end{align}
From Eqn.~\ref{disc_q_def_uni}
\begin{align}
y_{a} &\sim \N{\alpha_a\left(K \oh{x}{K} - \1{K}\right)}{\alpha_a K \I{K}},\\
y_{b} &\sim \N{\alpha_b\left(K \oh{x}{K} - \1{K}\right)}{\alpha_b K \I{K}}\\
\end{align}
and hence, from Identity~\ref{normal_identity_2}
\begin{align}
y_{a} + y_b &\sim \N{(\alpha_a+\alpha_b)\left(K \oh{\x}{KD} - \1{K}\right)}{(\alpha_a+\alpha_b) K \I{K}}.
\end{align}
Therefore, if $y$ is drawn from $\sender{\cdot}{x; \alpha_a + \alpha_b}$ and $\parst{i} = h(y, \parst{i-2})$ then $\parst{i}$ is drawn from\\ $\E_{\update(\parst{i-1} \mid \parst{i-2}, x; \alpha_a)} \update(\parst{i} \mid \parst{i-1}, x; \alpha_b)$ and
\begin{align}
\E_{\update(\parsnt{i-1}\mid \parsnt{i-2}, \x; \alpha_a)} \update(\parsnt{i} \mid \parsnt{i-1}, \x; \alpha_b) =  \update(\parsnt{i} \mid \parsnt{i-2}, \x; \alpha_a + \alpha_b),
\end{align}
as required. 
%%%%%%%%%%%%%%%%%%%%%%%%%%%%%%%%%%%%%%%%%%%%%%%%%%%%
\subsection{Accuracy Schedule \texorpdfstring{$\beta(t)$}{}}\label{sec:disc_beta}
As with continuous data, the guiding heuristic for $\beta(t)$ was to decrease the expected entropy of the input distribution linearly with $t$. In the continuous case, where the entropy is a deterministic function of $\sigma^2$, applying the heuristic was straightforward; in the discrete case an explicit computation of $\E_{\flow(\parsn \mid x; t)} H\left[\inp(\x \mid \parsn)\right]$ would be needed.
We were unable to derive an analytic expression for this term, but found that
\begin{align}
\beta(t) = t^2 \beta(1)\label{disc_beta_t}
\end{align}
was a reasonable approximation, with $\beta(1)$ determined empirically for each experiment.
Therefore
\begin{align}
\alpha(t) = \frac{d \beta(t)}{d t} = \beta(1) 2t.\label{disc_alpha_t}
\end{align}

\begin{figure}[t!]
\begin{centering}
\includegraphics[width=0.6\textwidth]{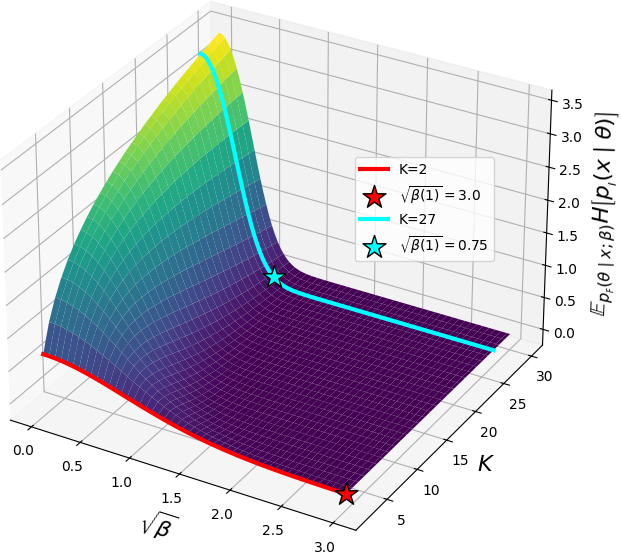}
\caption{\textbf{Accuracy schedule vs. expected entropy for discrete data}. The surface plot shows the expectation over the parameter distribution $p(\theta \mid x; \beta)$ of the entropy of the categorical input distribution $p(x \mid \theta)$ for $K=2$ to $30$ and $\sqrt{\beta}=0.01$ to $3$. The red and cyan lines highlight the entropy curves for 2 and 27 classes, the two values that occur in our experiments. The red and cyan stars show the corresponding values we chose for $\sqrt{\beta(1)}$.}
\end{centering}
\label{fig:disc_acc_vs_entropy}
\end{figure}
%%%%%%%%%%%%%%%%%%%%%%%%%%%%%%%%%%%%%%%%%%%%%%%%%%%%
\subsection{Bayesian Flow Distribution \texorpdfstring{$\flow(\cdot \mid \x; t)$}{}}
Substituting Eq.~\ref{disc_par_update_def} into Eq.~\ref{param_flow_dist},
\begin{align}
\flow(\parsn \mid \x; t) &= \E_{\N{\y \mid \beta(t)\left(K \oh{\x}{KD} - \1{KD}\right)}{\beta(t) K \I{KD}}} \delta\left(\parsn - \frac{e^{\y}\parsnt{0}}{\sum_{k=1}^K e^{\y_k}(\parsnt{0})_{k}}\right).
\end{align}
Since the prior is uniform with $\parsnt{0} = \vec{\frac{1}{K}}$, this reduces to
\begin{align}
\flow(\parsn \mid \x; t) &= \E_{\N{\y \mid \beta(t)\left(K \oh{\x}{KD} - \1{KD}\right)}{\beta(t) K \I{KD}}} \delta\left(\parsn - \text{softmax}(\y)\right),\label{disc_param_flow}
\end{align}
which can be sampled by drawing $\y$ from $\N{\beta(t)\left(K \oh{\x}{KD} - \1{KD}\right)}{\beta(t) K \I{KD}}$ then setting $\parsn = \text{softmax}(\y)$.

The sender distribution for discrete data can therefore be interpreted as a source of softmax logits for the Bayesian flow distribution; the higher the sender accuracy $\alpha$ is, the larger in expectation the logits corresponding to $\x$ will be in $\y$, hence the closer $\parsn$ will be to $\oh{\x}{KD}$ and the more information the network will gain about $\x$.
\begin{figure}[t!]
\includegraphics[width=\textwidth]{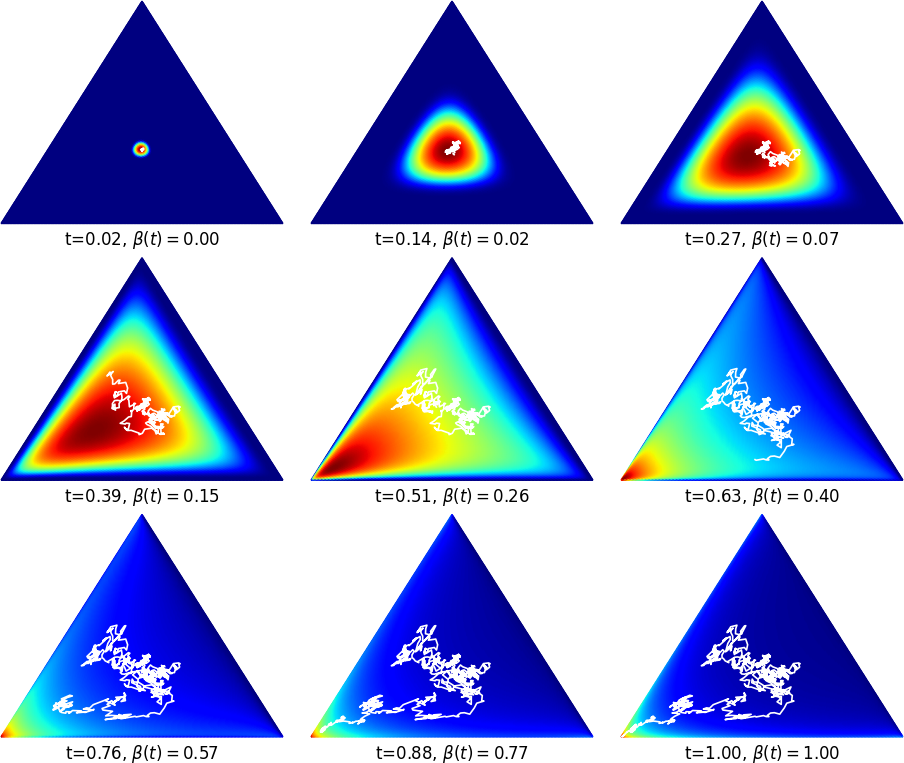}
\caption{\textbf{Bayesian flow for discrete data}. For $K=3$, the input distribution parameters $\parsn = (\theta_1, \theta_2, \theta_3)$ can be visualised as points on the 2-simplex, with the data $x$ corresponding to the bottom left corner. For the accuracy schedule $\beta(t)$ from Eq.~\ref{disc_beta_t}, the white line shows a single input parameter trajectory starting from $\parsnt{0} = \left(\frac{1}{3},\frac{1}{3},\frac{1}{3}\right)$ and evolving under the Bayesian update distribution $\update(\parsnt{i} \mid \parsnt{i-1}; x, \beta(t_i)-\beta(t_{i-1}))$ from Eq.~\ref{disc_par_update_def}, superimposed on log-scale heatmaps of the Bayesian flow distribution $\flow(\parsn \mid x; t)$ from Eq.~\ref{disc_param_flow}, plotted at regular intervals from $t=0.02$ to $1$.}
\label{fig:bayes_flow_disc}
\end{figure}
\begin{figure}[t!]
\includegraphics[width=\textwidth]{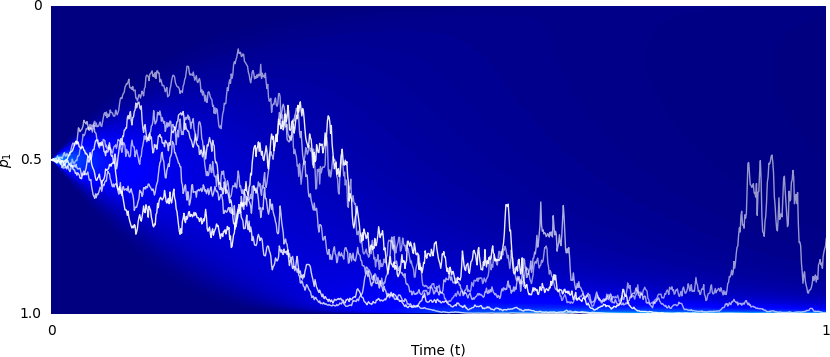}
\caption{\textbf{Bayesian flow for binary data}.
For the input probability $p_1$ of class one, the plot shows several parameter trajectories starting from $p_1 = 0.5$ at $t=0$ and evolving under the Bayesian update distribution to $t=1$, superimposed on a log-scale heatmap of the Bayesian flow distribution.
$\beta(1) = 4$ in this plot.
Note that both here and in Figure~\ref{fig:bayes_flow_disc} the convergence towards the data appears slower and noisier than the equivalent trajectories for continuous data in Figure~\ref{fig:cts_param_flow}. 
This is a fundamental consequence of discreteness: since all points in $\X$ are equidistant the input distributions cannot concentrate on values close to $\x$ as the trajectories progress.}
\label{fig:bayes_flow_bin}
\end{figure}
%%%%%%%%%%%%%%%%%%%%%%%%%%%%%%%%%%%%%%%%%%%%%%%%%%%%
\subsection{Reconstruction Loss \texorpdfstring{$L^r(\x)$}{}}\label{sec:disc_reconstruction}
The reconstruction loss for discrete data is
\begin{align}
L^r(\x) = -\E_{\flow(\parsn \mid \x, 1)}\ln \out(\x \mid \parsn; 1).\label{disc_reconstruction_loss}
\end{align}
%%%%%%%%%%%%%%%%%%%%%%%%%%%%%%%%%%%%%%%%%%%%%%%%%%%%
\subsection{Discrete-time Loss \texorpdfstring{$L^{n}(\x)$}{}}
From Eqs.~\ref{disc_q_def_uni} and \ref{disc_r_dist_uni},
\begin{align}
&\kl{\sender{\cdot}{\xdd{d}; \alpha}}{\rec^{(d)}(\cdot \mid \parsn; t, \alpha)}\\ 
&\qquad=\kl{\N{\alpha\left(K \oh{\xdd{d}}{K} - \1{K}\right)}{\alpha K \I{K}}}{\sum_{k=1}^K \out^{(d)}(k \mid \parsn; t) \N{\alpha\left(K\oh{k}{K}- \1{K}\right)}{\alpha K \I{K}}}.
\end{align}
Therefore, substituting into Eq.~\ref{disc_t_loss_exp},
\begin{align}
&L^{n}(\x) = n\E_{i \sim U\{1,n\},p(\parsn \mid \x ; t_{i-1}),\N{\y \mid \alphat{i}\left(K \oh{\x}{KD} - \1{KD}\right)}{\alphat{i} K \I{KD}}} \ln \N{\y \mid \alphat{i}\left(K \oh{\x}{KD} - \1{KD}\right)}{\alphat{i} K \I{KD}}\\
&\qquad\qquad\qquad-\sum_{d=1}^D \ln \left(\sum_{k=1}^K \out^{(d)}(k \mid \parsn; t_{i-1}) \N{\ydd{d} \mid \alphat{i}\left(K\oh{k}{K}- \1{K}\right)}{\alphat{i} K \I{K}}\right),\label{discdisc_t_loss_exp}
\end{align}
where, from Eq.~\ref{disc_beta_t},
\begin{align}
\alpha_i &= \beta(t_i) - \beta(t_{i-1})\\
&= \beta(1)\left(\left(\frac{i}{n}\right)^2 - \left(\frac{i-1}{n}\right)^2\right)\\
&= \beta(1)\left(\frac{2i -1}{n^2}\right).
\end{align}
%%%%%%%%%%%%%%%%%%%%%%%%%%%%%%%%%%%%%%%%%%%%%%%%%%%%
\subsection{Continuous-time Loss \texorpdfstring{$L^{\infty}(\x)$}{}}
Let
\begin{align}
\vv \defeq \frac{\y}{\alpha} + 1,
\end{align}
and apply Identity~\ref{normal_identity_1} to see that if
\begin{align}
\ydd{d} \sim \sender{\cdot}{\xdd{d}; \alpha} = \N{\alpha(K\oh{\xdd{d}}{K} - \1{K})}{\alpha K\I{K}},
\end{align}
then
\begin{align}
\didx{v}{d} \sim \N{K\oh{\xdd{d}}{K}}{\frac{K}{\alpha}\I{K}},
\end{align}
and similarly if
\begin{align}
\ydd{d} \sim \rec^{(d)}(\cdot \mid \parsn; t, \alpha) = \sum_{k=1}^K \out^{(d)}(k \mid \parsn; t) \N{\ydd{d} \mid \alpha\left(K\oh{k}{K}- \1{K}\right)}{\alpha K \I{K}},
\end{align}
then 
\begin{align}
\didx{v}{d} &\sim \sum_{k=1}^K \out^{(d)}(k \mid \parsn; t)\N{K\oh{k}{K}}{\frac{K}{\alpha}\I{K}}\\
&= K\sum_{k=1}^K \out^{(d)}(k \mid \parsn; t) \delta(\cdot - \oh{k}{K}) \ast \N{\0{K}}{\frac{K}{\alpha}\I{K}}.
\end{align}
The Kullback-Leibler divergence is invariant under affine transformations of variables, hence
\begin{align}
&\kl{\sender{\cdot}{\xdd{d}; \alpha}}{\rec^{(d)}(\cdot \mid \parsn; t, \alphat{i})}\\
&\qquad= \kl{\N{K\oh{\xdd{d}}{K}}{\frac{K}{\alpha}\I{K}}}{\sum_{k=1}^K \out^{(d)}(k \mid \parsn; t)K \delta(\cdot - \oh{k}{K}) \ast \N{\0{K}}{\frac{K}{\alpha}\I{K}}}.
\end{align}
Now set $C=K$, $g(\xdd{d}) = K \oh{\xdd{d}}{K}$ and
\begin{align}
P^{(d)}(\parsn, t) = K \sum_{k=1}^K \out^{(d)}(k \mid \parsn; t) \delta(\cdot - \oh{k}{K}),
\end{align}
which has finite variance and the following finite expectation
\begin{align}
E[P^{(d)}(\parsn, t)] = K \mathbf{\pred{e}}^{(d)}(\parsn, t),\label{disc_p_expectation}
\end{align}
where
\begin{align}
\mathbf{\pred{e}}^{(d)}(\parsn, t) \defeq \sum_{k=1}^K \out^{(d)}(k \mid \parsn; t)\oh{k}{K}.
\end{align}
The conditions in Eq.~\ref{convkl} are therefore satisfied and Eqs.~\ref{disc_p_expectation} and \ref{disc_alpha_t} can be substituted into Eq.~\ref{cts_t_loss} to yield
\begin{align}
L^{\infty}(\x) = K \beta(1) \E_{t\sim U(0,1),\flow(\parsn \mid \x, t)} t \|\oh{\x}{KD} - \mathbf{\pred{e}}(\parsn, t)\|^2,
\end{align}
where
\begin{align}
\mathbf{\pred{e}}(\parsn, t) \defeq \left(\mathbf{\pred{e}}^{(1)}(\parsn, t),\dots,\mathbf{\pred{e}}^{(D)}(\parsn, t)\right).
\end{align}
%%%%%%%%%%%%%%%%%%%%%%%%%%%%%%%%%%%%%%%%%%%%%%%%%%%%
\subsection{Pseudocode}
Pseudocode for evaluating the discrete-time loss $L^n(\x)$ and continuous-time loss $L^{\infty}(\x)$ for discrete data is presented in Algorithms~\ref{alg:n_step_loss_disc} and \ref{alg:cts_t_loss_disc}, while sample generation is presented in Algorithm~\ref{alg:samp_gen_disc}.
\begin{algorithm}[H]
\begin{algorithmic}
\Function{\lstinline{discrete_output_distribution}}{$\parsn \in [0,1]^{KD}, t \in [0,1]$}
\State Input $(\parsn, t)$ to network, receive $\net(\parsn, t)$ as output
\For{$d \in \ds{D}$}
\If{$k = 2$}
\State $\out^{(d)}(1 \mid \parsn; t) \gets \sigma\left(\net^{(d)}(\parsn, t)\right)$
\State $\out^{(d)}(2 \mid \parsn; t) \gets 1 - \out^{(d)}(1 \mid \parsn; t)$
\Else
\State $\out^{(d)}(\cdot \mid \parsn; t) \gets \text{softmax}(\net^{(d)}(\parsn, t))$
\EndIf
\EndFor
\State \textbf{Return} $\outn(\cdot \mid \parsn; t)$
\EndFunction
\end{algorithmic}
\end{algorithm}
\begin{algorithm}[H]
\caption{Discrete-Time Loss $L^{n}(\x)$ for Discrete Data}\label{alg:n_step_loss_disc}
\begin{algorithmic}
\State \textbf{Require:} $\beta(1) \in \R^+$, number of steps $n \in \mathbb{N}$, number of classes $K \in \mathbb{N}$
\State \textbf{Input:} discrete data $\x \in \ds{K}^D$
\State $i \sim U\{1, n\}$
\State $t \leftarrow (i-1)/n$
\State $\beta \leftarrow \beta(1)t^2$
\State $\y' \sim \N{\beta\left(K\oh{\x}{KD}-
\1{KD}\right)}{\beta K\I{KD}}$
\State $\parsn \gets \text{softmax}(\y')$
\State $\outn(\cdot \mid \parsn; t) \leftarrow \text{\sc{\lstinline{discrete_output_distribution}}}(\parsn, t)$
\State $\alpha \leftarrow \beta(1)\left(\frac{2i -1}{n^2}\right)$
\State $\y \sim \N{\alpha\left(K\oh{\x}{KD}-
\1{KD}\right)}{\alpha K\I{KD}}$
\State $ L^n(\x) \gets n \left[\ln \N{\y \mid \alpha\left(K\oh{\x}{KD}-
\1{KD}\right)}{\alpha K\I{KD}} - \sum_{d}\ln \left(\sum_{k} \out^{(d)}(k \mid \parsn; t) \N{\ydd{d} \mid \alpha\left(K\oh{k}{K}-
\1{K}\right)}{\alpha K\I{K}}\right)\right]$
\end{algorithmic}
\end{algorithm}
\begin{algorithm}[H]
\caption{Continuous-Time Loss $L^{\infty}(\x)$ for Discrete Data}\label{alg:cts_t_loss_disc}
\begin{algorithmic}
\State \textbf{Require:} $\beta(1) \in \R^+$, number of classes $K \in \mathbb{N}$
\State \textbf{Input:} discrete data $\x \in \ds{K}^D$
\State $t \sim U(0,1)$
\State $\beta \leftarrow \beta(1)t^2$
\State $\y \sim \N{\beta\left(K\oh{\x}{KD}-
\1{KD}\right)}{\beta K\I{KD}}$
\State $\parsn \gets \text{softmax}(\y)$
\State $\outn(\cdot \mid \parsn; t) \leftarrow \text{\sc{\lstinline{discrete_output_distribution}}}(\parsn, t)$
\State $\mathbf{\pred{e}}(\parsn, t) \gets \left(\sum_{k}\out^{(1)}(k \mid \parsn; t)\oh{k}{K},\dots,\sum_{k} \out^{(D)}(k \mid \parsn; t)\oh{k}{K}\right)$
\State $ L^{\infty}(\x) \gets K\beta(1)t\left\|\oh{\x}{KD} -\mathbf{\pred{e}}(\parsn, t) \right\|^2$
\end{algorithmic}
\end{algorithm}
\begin{algorithm}[H]
\caption{Sample Generation for Discrete Data}\label{alg:samp_gen_disc}
\begin{algorithmic}
\State \textbf{Require:} $\beta(1) \in \R^+$, number of steps $n \in \mathbb{N}$, number of classes $K \in \mathbb{N}$
\State $\parsn \gets \left(\vec{\frac{1}{K}}\right)$
\For{$i = 1$ to $n$} 
    \State $t \leftarrow \frac{i-1}{n}$
    \State $\k \sim \text{\sc{\lstinline{discrete_output_distribution}}}(\parsn, t)$
    \State $\alpha \leftarrow \beta(1)\left(\frac{2i -1}{n^2}\right)$
    \State $\y \sim \N{\alpha\left(K\oh{\k}{KD}-
    \1{KD}\right)}{\alpha K\I{KD}}$
    \State $\parsn' \gets e^{\y} \parsn$
    \State $\parsn \gets \frac{\parsn'}{\sum_k \parsn'_k}$
\EndFor
\State $\k \sim \text{\sc{\lstinline{discrete_output_distribution}}}(\parsn, 1)$
\State \textbf{Return} $\k$
\end{algorithmic}
\end{algorithm}
%%%%%%%%%%%%%%%%%%%%%%%%%%%%%%%%%%%%%%%%%%%%%%%%%%%%
\section{Experiments}\label{sec:experiments}
We evaluated Bayesian Flow Networks (BFNs) on the following generative benchmarks: CIFAR-10 (32$\times$32 8-bit color images), dynamically binarized MNIST (28$\times$28 binarized images of handwritten digits) and text8 (length 256 character sequences with a size 27 alphabet).
The continuous (Sec.~\ref{sec:cts}) and discretised (Sec.~\ref{sec:discretised}) versions of the system were compared on CIFAR-10, while the discrete version (Sec.~\ref{sec:discrete}) was applied to the other datasets.
In all cases, the network was trained using the continuous-time loss $L^{\infty}(\x)$, with the discrete-time loss $L^{n}(\x)$ evaluated for testing only, with various values of $n$.
Standard network architectures and training algorithms were used throughout to allow for direct comparison with existing methods.
Because the focus of this paper is on probabilistic modelling rather than image generation, FID scores were not calculated. However, examples of generated data are provided for all experiments.

\begin{table}[t!]
\centering
\begin{tabular}{@{}llc@{}}
\toprule
Model                                                               & \multicolumn{1}{c}{Dynamically Binarized MNIST} & CIFAR-10 \\ \midrule
Improved DDPM \citep{nichol2021improved}                            &                                     & 2.94     \\
NVAE \citep{vahdat2020nvae}                                         & \multicolumn{1}{c}{78.01}           & 2.91     \\
PixelVAE++\textsuperscript{\dag} \citep{sadeghi2019pixelvae++}        & \multicolumn{1}{c}{78.00}           & 2.90     \\
Locally Masked PixelCNN\textsuperscript{\dag} \citep{jain2020locally} & \multicolumn{1}{c}{77.58}           & 2.89     \\
Image Transformer\textsuperscript{\dag} \citep{parmar2018image}       &                                     & 2.89     \\
DDPM++ \citep{kim2021soft}                                          &                                     & 2.88     \\
LSGM \citep{vahdat2021score}                                        &                                     & 2.87     \\
VDVAE \citep{child2020very}                                         & \multicolumn{1}{c}{}                & 2.87     \\
Sparse Transformer\textsuperscript{\dag} \citep{child2019generating}  &                                     & 2.80     \\
Reflected Diffusion \citep{lou2023reflected}                        &                                     & 2.68     \\
VDM \citep{kingma2021variational}                                   &                                     & 2.65\\    
ARDM-Upscale 4 \citep{hoogeboom2021autoregressive}                  &                                     & 2.64
\\ \midrule
\textbf{BFN}                                                        & \multicolumn{1}{c}{77.87}                &      2.66    \\ 
\midrule
CR-NVAE* \citep{sinha2021consistency}                                & \multicolumn{1}{c}{76.93}          & 2.51    \\
VDM* \citep{kingma2021variational}                                   & \multicolumn{1}{c}{}                & 2.49    \\ \bottomrule
\end{tabular}
\caption{\textbf{Comparison of dynamically binarized MNIST and CIFAR-10 results with other methods}. The best published results for both datasets (*) use data augmentation for regularization. Results for models marked with (\textsuperscript{\dag}) are exact values; all other results are upper bounds.}
\label{tab:mnist-cifar-results}
\end{table}
%%%%%%%%%%%%%%%%%%%%%%%%%%%%%%%%%%%%%%%%%%%%%%%%%%%%
\subsection{Dynamically Binarized MNIST}
\begin{table}[t!]
\centering
\begin{tabular}{cccccccc}
\toprule
$n$-steps & 10 & 25 & 50 & 100 & 784 & 1000 & $\infty$\\ 
\midrule
NPI & $95.21$ & $84.40$ & $81.06$ & $79.46$ & $78.02$ & $78.07$ & $77.87$ \\ 
\bottomrule
\end{tabular}
\caption{\textbf{Dynamically binarized MNIST results}. NPI is nats per image  averaged over 2,000 passes through the test set with $L^{n}(\x)$ or $L^{\infty}(\x)$ sampled once per test image per pass. The reconstruction loss $L^r(\x)$ (included in NPI) was $0.46$. 784 is the total number of pixels per image, hence the number of steps required to generate an image with an autoregressive model.}
\label{tab:mnist_results}
\end{table}

\textbf{Data.}\quad 
The binarized MNIST benchmark data was originally created from the MNIST dataset of handwritten images \citep{lecun-mnisthandwrittendigit-2010} by treating the grayscale pixel intensities as Bernoulli probabilities and sampling a particular binarization \citep{salakhutdinov2008quantitative} which is held fixed during training.
In recent years, a variant of the same benchmark has become more popular, with a new binarization sampled from the probabilities for every training batch.
The two are not comparable, as the latter, which we refer to as dynamically binarized MNIST, effectively has a larger training set and hence gives better test set performance.
All our experiments and the results referenced from the literature use dynamically binarized MNIST.
\\

\begin{figure}[t!]
\centering
\begin{subfigure}{.49\textwidth}
  \centering
  \includegraphics[width=0.9\linewidth]{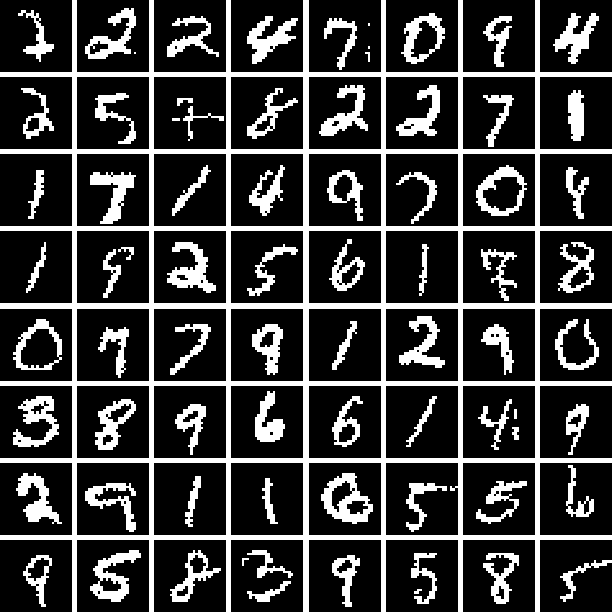}
  \caption{Test Data}
\end{subfigure}
\begin{subfigure}{.49\textwidth}
  \centering
  \includegraphics[width=0.9\linewidth]{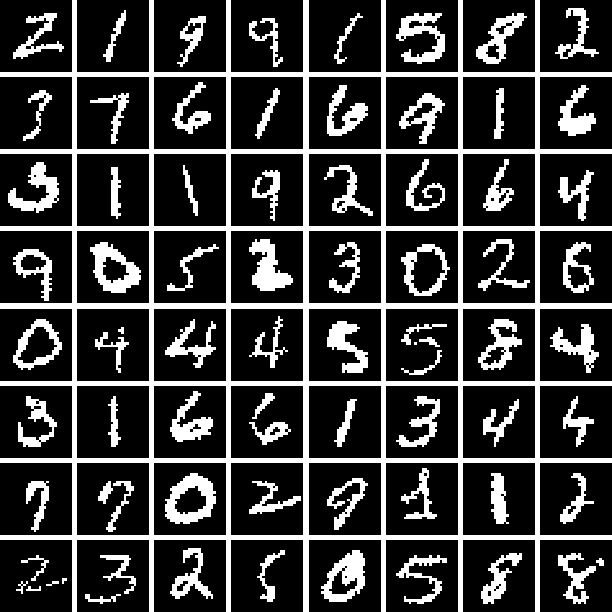}
  \caption{Generated Data}
\end{subfigure}
\caption{\textbf{MNIST real and generated data}. Samples generated with 100 steps.}
\end{figure}

\noindent\textbf{Setup.}\quad The network architecture was based on a U-Net introduced for diffusion models~\citep{nichol2021improved}.
Starting from the hyperparameters used for the CIFAR-10 dataset (see Appendix A in the above reference), we made the following modifications: the number of resblocks was reduced from three to two and the layer widths were reduced from $[C,2C,2C,2C]$ to $[C,2C,2C]$ with $C=128$.
Finally, the input and output of the standard network were concatenated and projected back to the output size.
600 randomly selected training images (1\% of the training set) were used as a validation set.
The optimiser was AdamW~\citep{loshchilov2017decoupled} with learning rate $0.0001$, weight decay 0.01 and $(\beta_1,\beta_2) = (0.9,0.98)$.
Dropout was used with probability 0.5, the training batch size was 512, and $\beta(1)$ was set to $3$ (see Sec.~\ref{sec:disc_beta}).
The network was trained for $150\,000$ weight updates until early stopping.
An exponential moving average of model parameters with a decay rate of 0.9999 was used for evaluation and sample generation.
The total number of learnable parameters was approximately 25M.
\\

\begin{figure}[t!]
\centering
\begin{subfigure}{\textwidth}
    \centering
    \includegraphics[width=\linewidth]{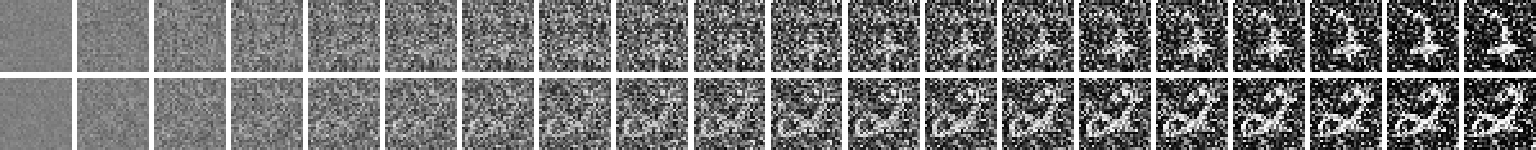}
    \caption{Input Distribution}
\end{subfigure}
\par\bigskip
\begin{subfigure}{\textwidth}
    \centering
    \includegraphics[width=\linewidth]{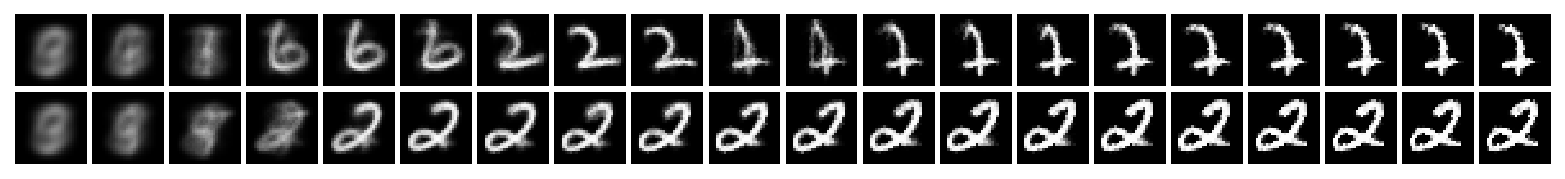}
    \caption{Output Distribution}
\end{subfigure}

\caption{\textbf{MNIST Input and output distributions}. For two test set images the figure shows the white pixel probability at 20 steps evenly spaced between $t=0$ and $t=1/3$. Note how the input probabilities are initially uniform whereas the output distribution initially predicts a superposition of multiple digits, closely matching the per-pixel marginal prior over the training set: this supports our belief that the network learns to correct for the uniform prior in the input distribution. Also note that the output distribution is much less noisy than the input distribution, and that it changes more dramatically as new information is received (e.g. the network appears to switch from predicting a $6$ to a $2$ to a $7$ for the first image). This highlights the network's use of context to resolve ambiguity and noise in the input distribution.}
\end{figure}

\noindent\textbf{Results.}\quad As can be seen from Table~\ref{tab:mnist-cifar-results}, BFN is close to state-of-the-art for this task with no data augmentation. 
Table~\ref{tab:mnist_results} shows the expected inverse relationship between loss and number of steps.
Direct optimisation of the $n$-step loss would likely lead to reduced loss for low values of $n$; however we leave that for future work.
One issue is that the reconstruction loss was relatively high at 0.46 nats per image. 
The obvious way to decrease this would be to increase $\beta(1)$, but we found that doing so led to slower learning and worse performance.
Along with the loss curves in Figure~\ref{fig:bin_mnist_loss}, this suggests that the accuracy schedule is suboptimal for binary data.

\begin{figure}[t!]
\centering
\begin{subfigure}{.49\textwidth}
  \centering
  \includegraphics[width=0.9\linewidth]{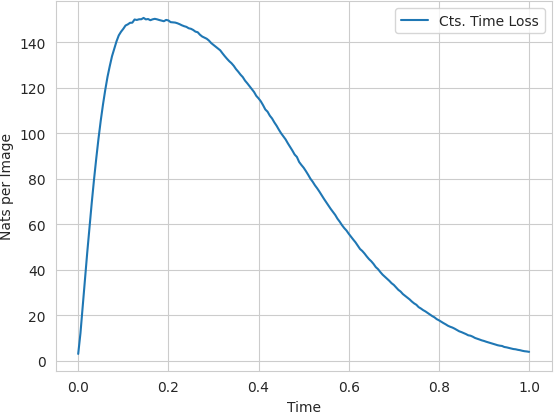}
\end{subfigure}
\begin{subfigure}{.49\textwidth}
  \centering
  \includegraphics[width=0.9\linewidth]{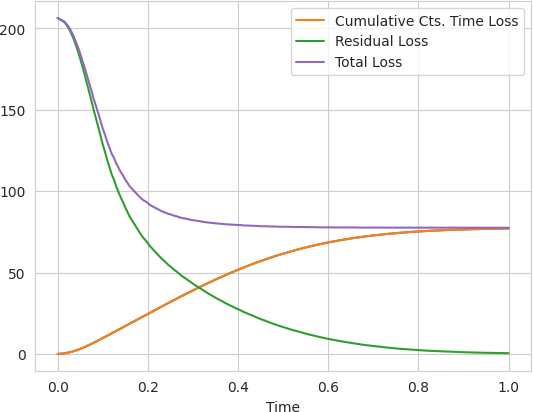}
\end{subfigure}
\caption{\textbf{MNIST losses against time}. The left plot shows the mean over the test set of the cts. time loss $L^{\infty}(\x)$ used for training for transmission time $t$ between 0 and 1. The right plot shows the average cumulative value of $L^{\infty}(\x)$ up to $t$, along with the reconstruction loss $L^r(\x)$ evaluated at $t$ and the sum of these two losses, which would be the total loss if the transmission process halted at $t$.
Note the unevenness of $L^{\infty}(\x)$ against $t$: we speculate that rescaling $\beta(t)$ to make the loss curve more uniform could improve performance.}
\label{fig:bin_mnist_loss}
\end{figure}
%%%%%%%%%%%%%%%%%%%%%%%%%%%%%%%%%%%%%%%%%%%%%%%%%%%%
\subsection{CIFAR-10}
\begin{table}[t!]
\centering
\begin{tabular}{ccccc}
\toprule
$n$-steps & Cts. (256 bins) & Discd. (256 bins) & Cts. (16 bins) & Discd. (16 bins)\\ 
\midrule
10  & 6.18 & 3.91 & 1.42 & 1.16\\ 
25  & 3.65 & 3.16 & 1.11 & 1.02\\ 
50  & 3.10 & 2.93 & 1.03 & 0.98\\ 
100 & 2.86 & 2.81 & 0.99 & 0.96 \\ 
250 & 2.73 & 2.73 & 0.97 & 0.94\\ 
500 & 2.69 & 2.71 & 0.96 & 0.94\\ 
1000& 2.67 & 2.70 & 0.96 & 0.94\\ 
\midrule
$\infty$ &  2.66 & 2.68 & 0.96 & 0.94\\ 
\bottomrule
\toprule
$L^r(\x)$ &  0.001 & 0.003 & 0.073 & 0.070\\ 
\midrule
Updates & 5M & 5M & 250K & 1M \\ 
\bottomrule
\end{tabular}
\caption{\textbf{CIFAR-10 results}. All losses are bits per dimension (BPD) averaged over 100 passes through the test set with $L^{n}(\x)$ or $L^{\infty}(\x)$ sampled once per test image per pass. The reconstruction losses $L^r(\x)$ (included in BPD) and the number of training updates for each network are shown below.}
\label{tab:cifar_results}
\end{table}

\textbf{Data.}\quad Two sets of generative modelling experiments were conducted on the CIFAR-10 database~\citep{Krizhevsky09learningmultiple}, one at the standard bit-depth of 8, corresponding to 256 discretised bins per colour channel, and one at a reduced bit-depth of 4, corresponding to $16$ bins per channel.
In both cases the bins evenly partitioned the interval $[-1,1]$ and the data was pre-processed by assigning each channel intensity to the nearest bin centre, as described in Section~\ref{sec:discretised}.
The purpose of comparing 16 and 256 bin discretisation was twofold: (1) to test the hypothesis that the advantage of training with the discretised loss from Section~\ref{sec:discretised} rather than the continuous loss from Section~\ref{sec:cts} would be greater when the number of bins was lower, and (2) to test whether modelling the data at lower precision would lead to improved perceptual quality.
No data augmentation, such as horizontal flips or random crops, was used on the training set.
\\

\begin{figure}[t!]
\centering
\begin{subfigure}{.5\textwidth}
  \centering
  \includegraphics[width=0.9\linewidth]{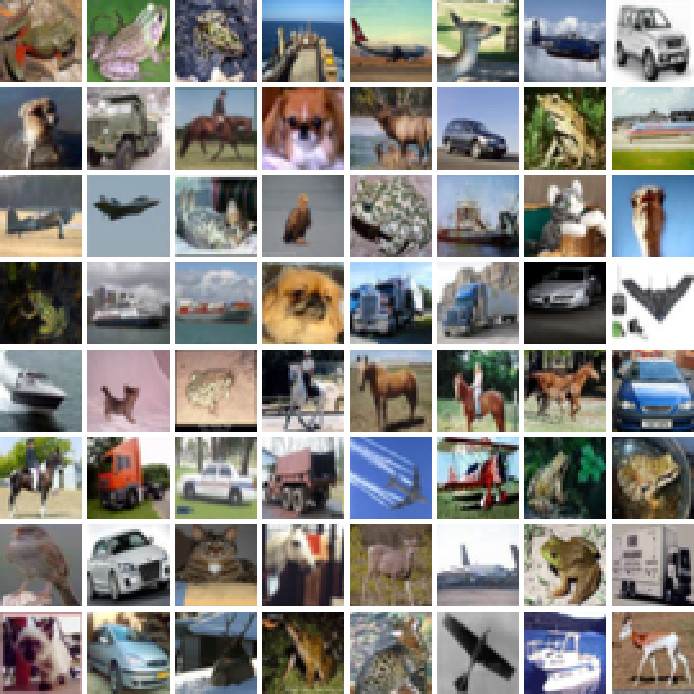}
  \caption{Test Data (256 bins)}
\end{subfigure}%
\begin{subfigure}{.5\textwidth}
  \centering
  \includegraphics[width=0.9\linewidth]{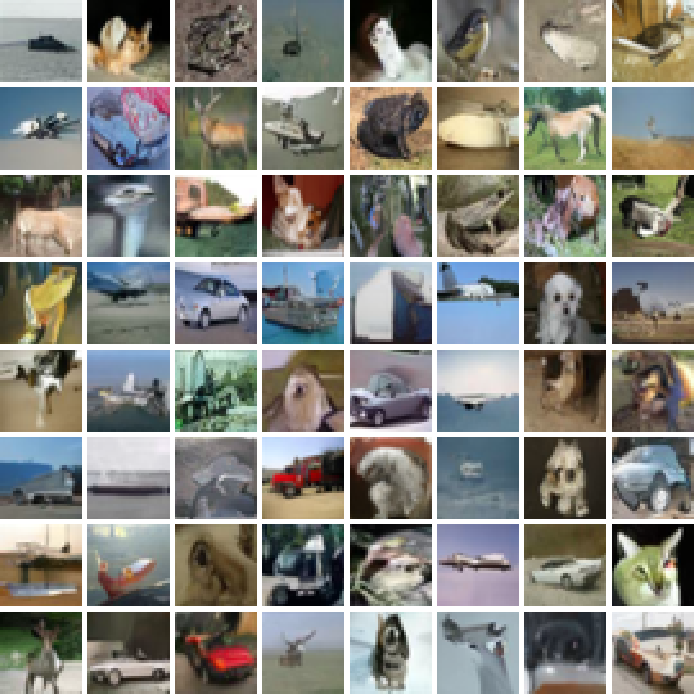}
  \caption{Generated Data (256 bins)}
\end{subfigure}
\par\bigskip
\begin{subfigure}{.5\textwidth}
  \centering
  \includegraphics[width=0.9\linewidth]{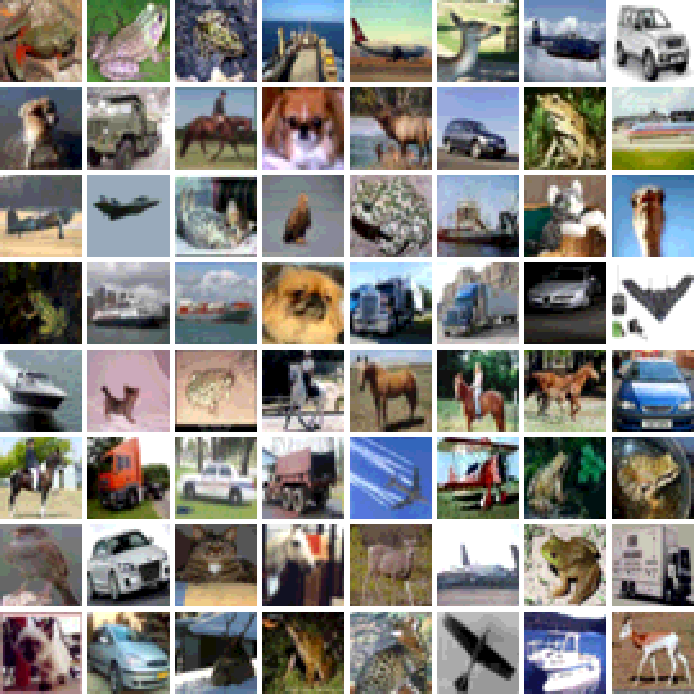}
  \caption{Test Data (16 bins)}
\end{subfigure}%
\begin{subfigure}{.5\textwidth}
  \centering 
  \includegraphics[width=0.9\linewidth]{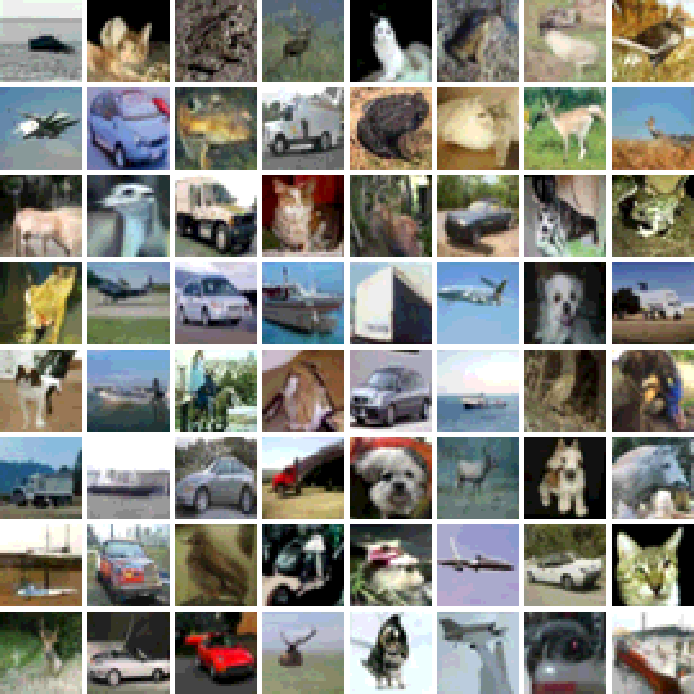}
  \caption{Generated Data (16 bins)}
\end{subfigure}
\caption{\textbf{CIFAR-10 real and generated data}. Samples generated with 4,000 steps, using networks trained with discretised loss. The same random seed was used for both sets of samples. Note the improved image quality of the 16 bin samples compared to the 256 bin samples.}
\label{fig:cifar_samples}
\end{figure}

\noindent\textbf{Setup.}\quad 
The network architecture was essentially the same as that used for Variational Diffusion Models (VDMs~\citep{kingma2021variational}), including the Fourier feature inputs.
The only modification was an extra input-output connection similar to the network for MNIST.
In total there were approximately 31M learnable parameters.
The following hyperparameters were used for all CIFAR-10 experiments:
a validation set of 500 randomly selected training images (1\% of the training set),
the  AdamW~\citep{loshchilov2017decoupled} optmizer with weight decay 0.01, learning rate $0.0002$ and $(\beta_1,\beta_2) = (0.9,0.99)$,
dropout with probability 0.1,
training batch size of 128,
$t_{min} = 1\mathrm{e}{-6}$,
$[x_{min}, x_{max}] = [-1, 1]$, and
an exponential moving average of model parameters with a decay rate of 0.9999 for evaluation and sample generation.
For the 256 bin experiments $\sigma_1 = 0.001$, while for the 16 bin experiments $\sigma_1 = \sqrt{0.001}$.
For the networks trained with continuous loss, the reconstruction loss was measured using the discretised version of $L^r(\x)$ from Section~\ref{sec:discd_reconstruction} rather than the continuous version from Section~\ref{sec:cts_reconstruction}, using a discretised Gaussian with mean equal to $\hat{x}(\parsn, 1)$ and std.\ deviation chosen empirically to be $\sigma_1$ for 256 bins and $0.7 \sigma_1$ for 16 bins.
This ensured the results were comparable between continuous and discretised training, and consistent with the literature.
\\

\begin{figure}[t!]
\centering
\begin{subfigure}{\textwidth}
    \centering
    \includegraphics[width=\linewidth]{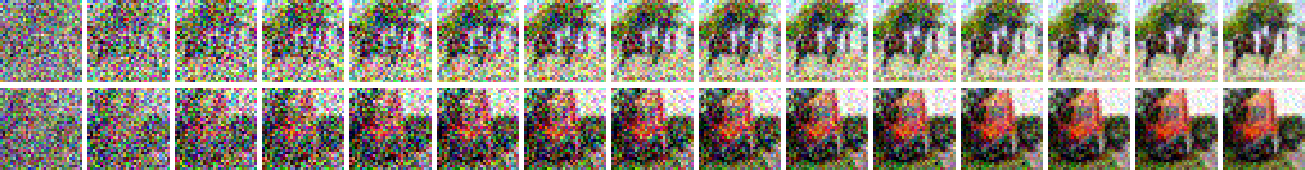}
    \caption{Input Mean}
\end{subfigure}
\par\bigskip
\begin{subfigure}{\textwidth}
    \centering
    \includegraphics[width=\linewidth]{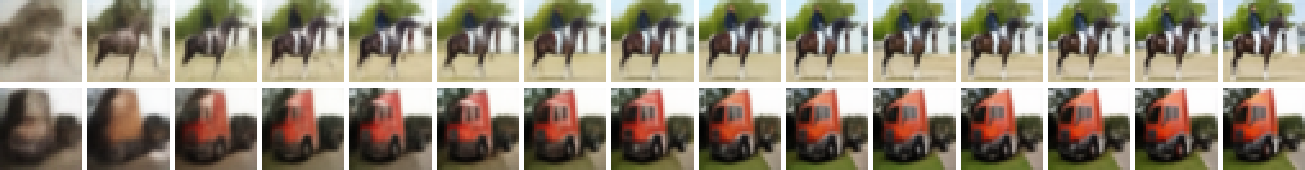}
    \caption{Output Mean}
\end{subfigure}
\caption{\textbf{CIFAR-10 Input and output distributions}. For two test set images the figure shows the means of the input and output distributions at steps evenly spaced between $t=0$ and $t=0.25$. }
\end{figure}

\noindent\textbf{Results.}\quad Table~\ref{tab:mnist-cifar-results} shows that the best performing BFN gives 2.66 BPD for the 256 bin data, which is close to the state-of-the-art at 2.64 BPD.
The most obvious performance benchmark (given the shared network architecture and similarity in loss function) is the VDM result at 2.65 BPD~\citep{kingma2021variational}.
However this took 10M weight updates to achieve, and due to time constraints we were only able to train BFNs for 5M updates.
Validation performance was still improving after 5M updates, and it remains unclear how much performance would improve with 10M updates.

Table~\ref{tab:cifar_results} shows that discretised loss gave better performance than continuous loss for 16 bins, as well as much faster training time (250K updates vs. 1M).
This supports the hypothesis that training with discretised loss is most beneficial when the number of bins is relatively low.
Furthermore, for both 16 and 256 bins, discretised training gave much better results when the number of steps $n$ was low (e.g. 10 or 25).
However continuous loss gave better performance than discretised loss on 256 bins (2.66 BPC vs 2.68); more investigation would be needed to understand why.

Figure~\ref{fig:cifar_samples} shows that discretised training with 16 bins gives better sample quality than training with 256 bins.
This is presumably because the loss function of the former is restricted to the first four bits of the data in which --- as can be seen by comparing the test data at 16 and 256 bins --- most of the perceptually relevant information is contained.
An interesting direction for future work would be to train one BFN to model the lower bits of an image, and a second BFN to conditionally upscale to higher bits, as has previously been explored for autoregressive models~\citep{menick2018generating,hoogeboom2021autoregressive}.

\begin{figure}[t!]
\centering
\begin{subfigure}{.5\textwidth}
  \centering
  \includegraphics[width=0.9\linewidth]{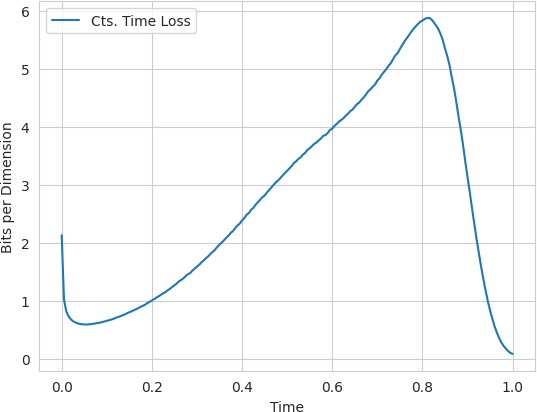}
\end{subfigure}%
\begin{subfigure}{.5\textwidth}
  \centering
  \includegraphics[width=0.9\linewidth]{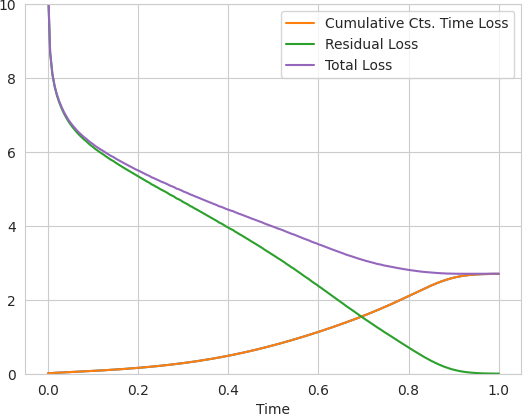}
\end{subfigure}
\caption{\textbf{CIFAR-10 losses against time}. The plot was made using the network trained with discretised loss on 256 bins. Note the high loss at the very start of the process, which we did not observe with discrete data.}
\end{figure}

\begin{table}[t!]
\centering
\begin{tabular}{@{}lll@{}}
\toprule
                                                 & Model                                             & BPC                    \\ \midrule
\multirow{3}{*}{Flow-based models}               & IAF/SCF\textsuperscript{\dag} \citep{ziegler2019}                       & 1.88                   \\
                                                 & Argmax Coupling Flow\textsuperscript{\dag} \citep{hoogeboom2021}        & 1.80                   \\
                                                 & Discrete Flow\textsuperscript{\dag} \citep{tran2019}                    & 1.23                   \\ \midrule
\multirow{3}{*}{Order-agnostic Models}           & OA-ARDM \citep{hoogeboom2021autoregressive}       & 1.43 $\pm$ 0.001 \\
                                                 & MAC \citep{shih2022training}           & 1.40 \\
                                                 \midrule
\multirow{3}{*}{Diffusion models}                & Multinomial Diffusion \citep{hoogeboom2021}                   & 1.72 \\
& D3PM uniform \citep{austin2021d3pm}                   & 1.61 $\pm$ 0.02 \\
                                                 & D3PM NN \citep{austin2021d3pm}                        & 1.59 $\pm$ 0.03 \\
                                                 & D3PM mask \citep{austin2021d3pm}                      & 1.45 $\pm$ 0.02 \\ \midrule
                                                 & \textbf{BFN}   & \textbf{1.41}                       \\ \midrule
Autoregressive baseline                          & Transformer\textsuperscript{\dag} \citep{austin2021d3pm}                    & 1.23                   \\
Best result*                & Adaptive Span Transformer\textsuperscript{\dag} \citep{sukhbaatar2019} & 1.07                   \\ \bottomrule
\end{tabular}
\caption{\textbf{Comparison of text8 results with other methods}. The best published model on this dataset (*) was trained on sequences of length 512. Rest of the above models were trained on sequences of length 256. Results for models marked with (\textsuperscript{\dag}) are exact values; all other results are upper bounds.
}
\label{tab:text8_comparison}
\end{table}
%%%%%%%%%%%%%%%%%%%%%%%%%%%%%%%%%%%%%%%%%%%%%%%%%%%%
\subsection{text8}
\begin{table}[t!]
\centering
\begin{tabular}{cccccccc}
\toprule
$n$-steps & 10 & 25 & 50 & 100 & 256 & 1000 & $\infty$\\ 
\midrule
BPC & 1.70 & 1.52 & 1.47 & 1.43 & 1.42 & 1.41 & 1.41 \\ 
\bottomrule
\end{tabular}
\caption{\textbf{text8 results}. BPC is bits per character averaged over 1M randomly cropped sequences from the test set with $L^{n}(\x)$ or $L^{\infty}(\x)$ sampled once per crop. The reconstruction loss $L^r(\x)$ (included in BPC) was $0.006$.}
\label{tab:text8_results}
\end{table}

\noindent\textbf{Data.}\quad The text8 dataset~\citep{mahoney09ltcb} was derived from a subset of the enwik9 Wikipedia dataset by removing punctuation and restricting the text to lowercase Latin letters and spaces, giving an alphabet of size 27.
For clarity, we represent the space character with an underscore in figures.
\\

\begin{figure}[t!]
\centering
\begin{subfigure}{.5\textwidth}
  \centering
  \includegraphics[width=0.9\linewidth]{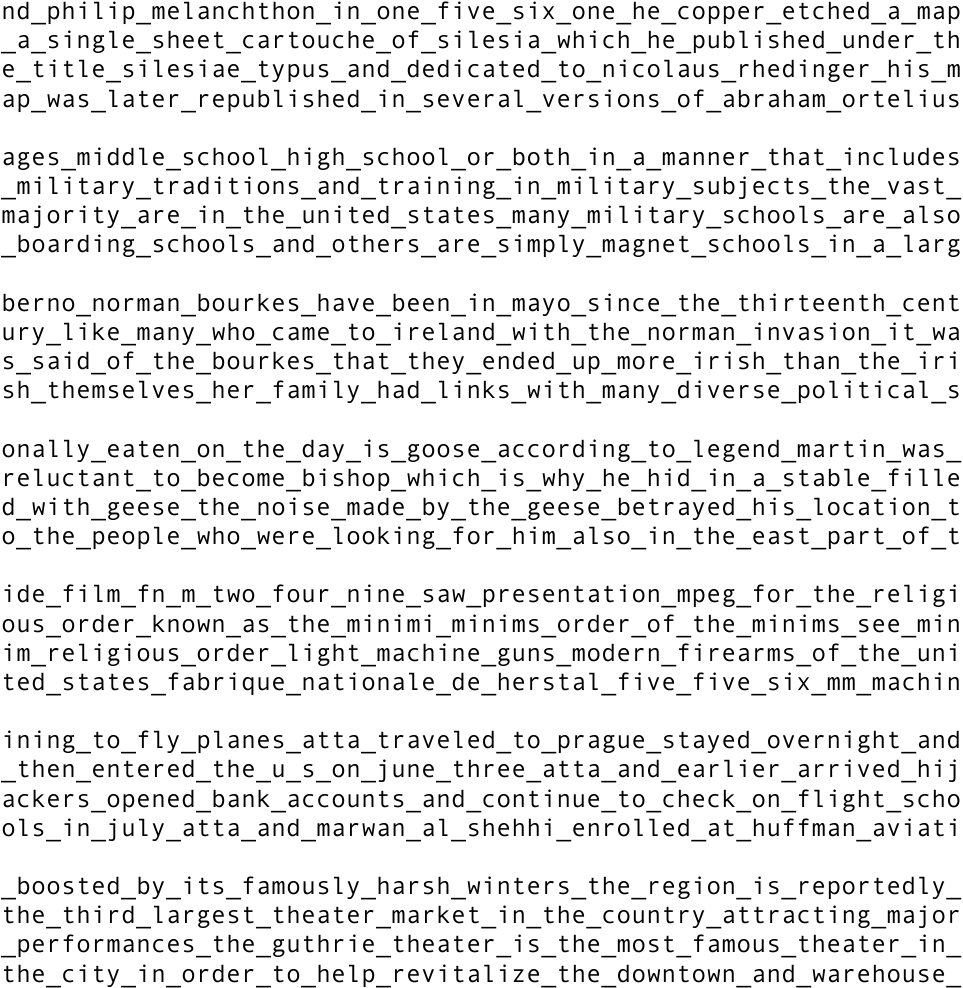}
  \caption{Test Data}
\end{subfigure}%
\begin{subfigure}{.5\textwidth}
  \centering
  \includegraphics[width=0.9\linewidth]{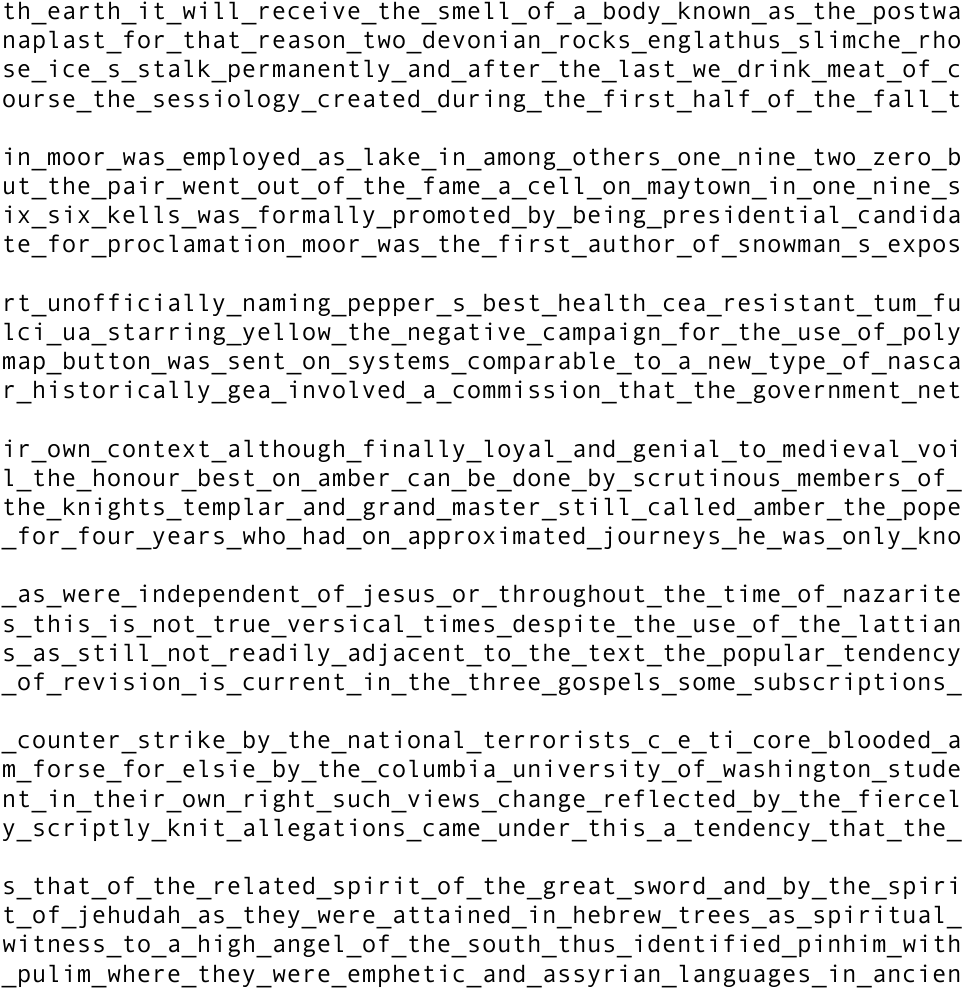}
  \caption{Generated Data}
\end{subfigure}
\caption{\textbf{text8 real and generated data.} Samples generated with 1000 steps.}
\end{figure}

\noindent\textbf{Setup.}\quad The network architecture was a Transformer similar to the small model ($d_{\text{model}}=768$) used by \citet{radford2019language} except that it uses the GELU activation function \citep{hendrycks2016gaussian} and the depth was increased to 24 layers.
The input and output of the Transformer were concatenated and then projected back to the output size to produce the final output.
The standard training/validation/test split of 90M/5M/5M consecutive characters was used, and
the network was trained with a batch size of 3328 sequences of length 256, randomly cropped from the training set, for 1.2\,M weight updates using the AdamW optimizer\citep{loshchilov2017decoupled}. 
The learning rate was set to $10^{-4}$, weight decay to 0.1 and $(\beta_1, \beta_2)$ to $ (0.9, 0.98)$.
An exponential moving average of model parameters with a decay rate of 0.9999 was used for evaluation and sample generation.
Dropout was not used, but overfitting was observed towards the end of training indicating that regularization may further improve results.
$\beta(1)$ was 0.75.
The total number of learnable parameters was approximately 170M.
Note that the  batch size and number of layers were larger than prior results from diffusion models. 
The first choice increases model capacity while the second tends to make overfitting more likely.
These choices were made to maximize the utilization of available resources while achieving results in reasonable time.
\\

\begin{figure}[t!]
\centering
\includegraphics[width=\linewidth]{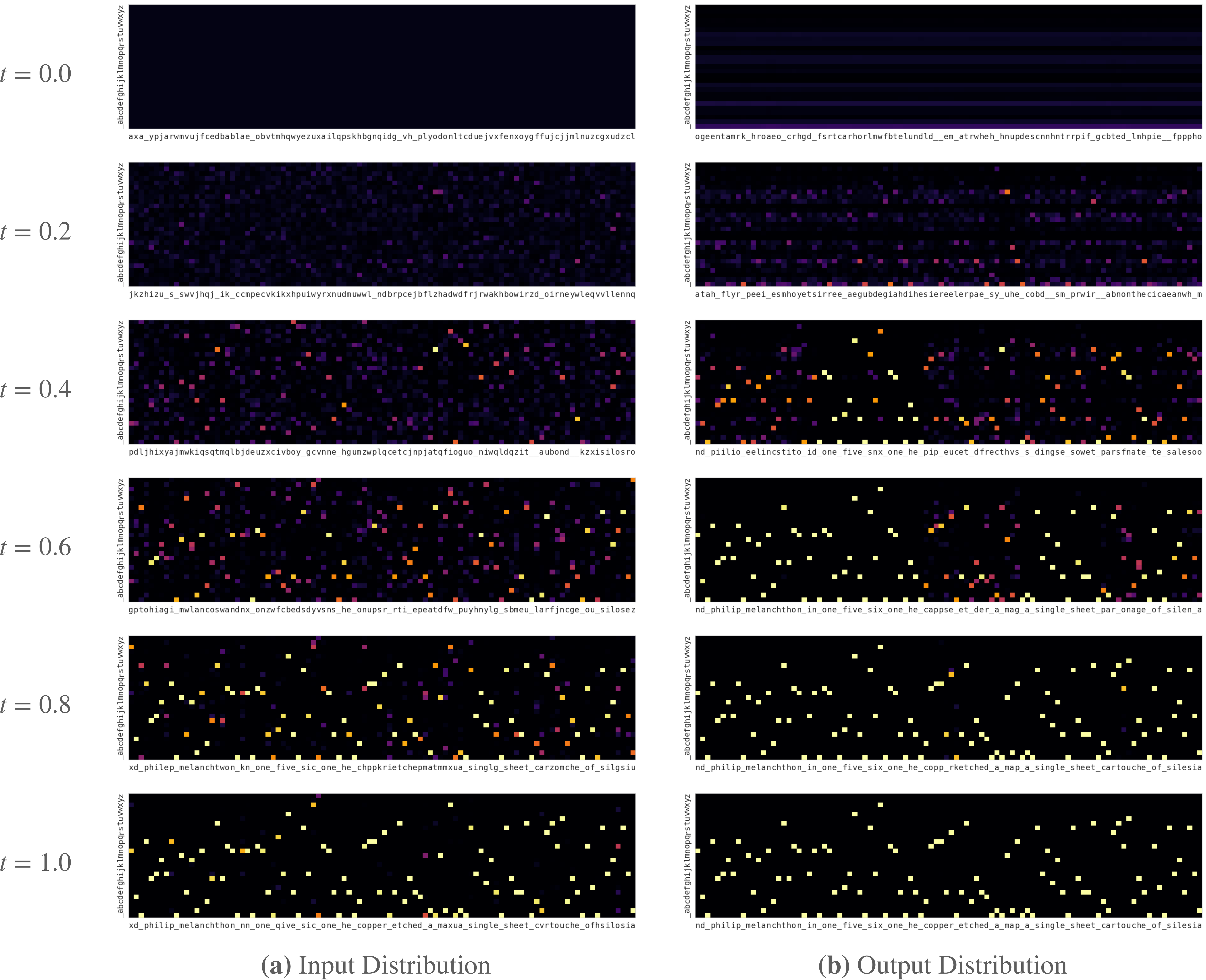}
\caption{\textbf{text8 Input and Output Distributions}. The heatmaps show the character probability distributions across part of a test sequence at various times during the flow process. Whereas the expected entropy for each letter decreases independently in the input distribution, the entropy of the output distribution tends to chunk into words and phrases --- e.g. the date ``one\_five\_six\_one'' is confidently predicted early in the process.}
\end{figure}

\noindent\textbf{Results.}\quad
Table~\ref{tab:text8_comparison} shows that BFN yielded a  1.41 BPC on the text8 test set, which is better than all discrete diffusion models we found in the literature, and close to the best order-agnostic model, MAC at 1.40 BPC.
We note however that both a standard autoregressive baseline and a discrete flow model perform substantially better at 1.23 BPC.
Table~\ref{tab:text8_results} shows that performance is reasonably robust to decreased $n$, with only 100 steps required to reach 1.43 BPC.
This result could probably be improved by training with the discrete-time loss.

\begin{figure}[t!]
\centering
\includegraphics[width=\linewidth]{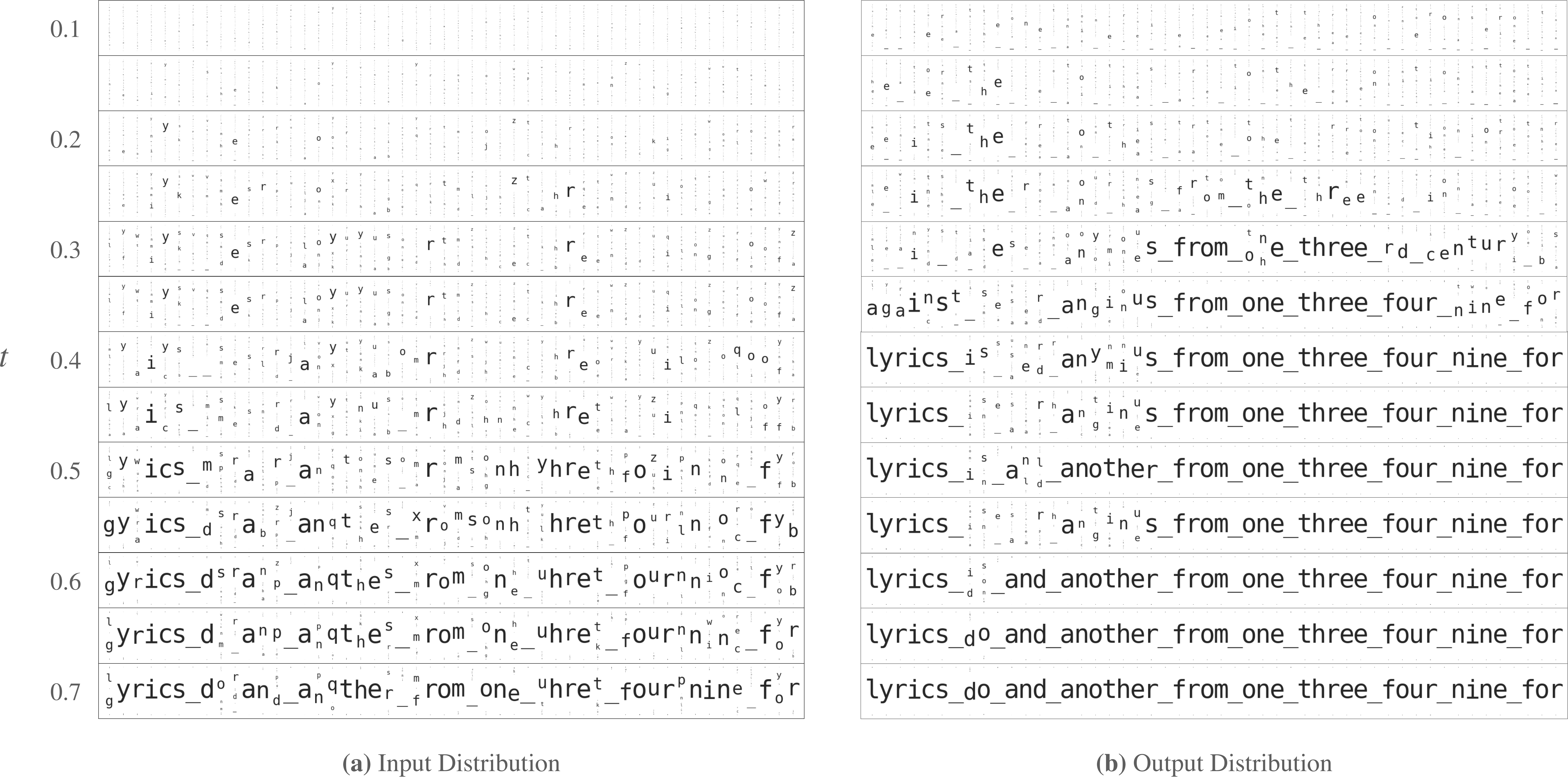}
\caption{\textbf{text8 Input and Output Distributions}. An alternative visualisation with the character sizes scaled in proportion to their probability.}
\end{figure}
%%%%%%%%%%%%%%%%%%%%%%%%%%%%%%%%%%%%%%%%%%%%%%%%%%%%
\section{Conclusion}\label{sec:conclusion}
This paper introduced Bayesian Flow Networks, a new class of generative model that combines Bayesian inference with neural networks in an iterative modelling process.
Discrete and continuous-time loss functions were derived along with sampling procedures, and the model was succesfully applied to continuous, discretised and discrete data.
We hope this work will inspire fresh perspectives and new directions for generative modelling research.
%%%%%%%%%%%%%%%%%%%%%%%%%%%%%%%%%%%%%%%%%%%%%%%%%%%%
\section*{Ackowledgements}\label{sec:acknowledgements}
We would like to thank Vojtech Micka for his invaluable engineering and infrastructure support.
%%%%%%%%%%%%%%%%%%%%%%%%%%%%%%%%%%%%%%%%%%%%%%%%%%%%
\bibliographystyle{plainnat}
\bibliography{bibliography}
%%%%%%%%%%%%%%%%%%%%%%%%%%%%%%%%%%%%%%%%%%%%%%%%%%%%
\end{document}